\documentclass[10pt,conference]{IEEEtran}
\usepackage{times}
\usepackage[utf8]{inputenc}
\usepackage[english]{babel}
\usepackage{url}
\usepackage{epsfig}
\usepackage{epstopdf}
\usepackage{graphicx}
\usepackage{amsmath}
\usepackage{amsthm}
\usepackage{amssymb}
\usepackage{ mathrsfs }
\usepackage{subcaption}
\usepackage{color}
\usepackage{algpseudocode}
\usepackage{algorithm}
\usepackage{placeins}
\usepackage{rotating}
\usepackage{picins}
\usepackage{capt-of,etoolbox}
\usepackage[pagebackref=true,breaklinks=true,letterpaper=true,colorlinks,bookmarks=false]{hyperref}

\usepackage{ulem} % for \sout -- strikethrough

\usepackage[font=small,labelfont=bf,tableposition=top]{caption}

\DeclareMathOperator{\tr}{tr}
\DeclareMathOperator{\prox}{prox}

\DeclareMathOperator{\sign}{sgn}
\newcommand{\G}{\mathcal{G}}
\newcommand{\K}{\mathcal{K}}

\newcommand{\E}{\mathcal{E}}

\newcommand{\Rbb}{\mathbb{R}}
\newcommand{\Larg}{\mathcal{L}}

\newcommand{\n}[1]{{\textcolor[rgb]{0,0,1}{#1}}}
\newcommand{\norm}[1]{\ensuremath{\| #1\|}}
\newtheorem{thm}{Theorem}
\newtheorem{defn}{Definition}
\newtheorem{lemma}{Lemma}

\newcommand\blfootnote[1]{%
  \begingroup
  \renewcommand\thefootnote{}\footnote{#1}%
  \addtocounter{footnote}{-1}%
  \endgroup
}
\begin{document}

\title{Compressive PCA for Low-Rank Matrices on Graphs}

\author{Nauman Shahid$^{*}$, Nathanael Perraudin,  Gilles Puy$^\dagger$, Pierre Vandergheynst\\
Email: \{nauman.shahid, nathanael.perraudin,  pierre.vandergheynst\}@epfl.ch, $\dagger$ gilles.puy@gmail.com
}

\maketitle
\begin{abstract} 
\blfootnote{Affiliation: Signal Processing Laboratory 2 (LTS2), EPFL STI IEL, Lausanne, CH-1015, Switzerland. Phone: +41 21 69 34754.  $^\dagger$ G. Puy contributed to this work while he was at INRIA (Rennes - Bretagne Atlantique, Campus de Beaulieu, FR-35042 Rennes Cedex, France). N.Shahid and N.Perraudin are supported by the SNF grant no. 200021\_154350/1 for the project ``Towards signal processing on graphs''.G.Puy was funded by the European Research Council, PLEASE project (ERC-StG-2011-277906).} We introduce a novel framework for an approximate recovery of data matrices which are low-rank on graphs, from sampled measurements. The rows and columns of such matrices belong to the span of the first few eigenvectors of the graphs constructed between their rows and columns. We leverage this property to recover the non-linear low-rank structures efficiently from sampled data measurements, with a low cost (linear in $n$).  First, a Resrtricted Isometry Property (RIP) condition is introduced for efficient uniform sampling of the rows and columns of such matrices based on the cumulative coherence of graph eigenvectors. Secondly, a state-of-the-art fast low-rank recovery method is suggested for the sampled data. Finally, several efficient, parallel and parameter-free decoders are presented along with their theoretical analysis for decoding the low-rank and cluster indicators for the full data matrix. Thus, we overcome the computational limitations of the standard \textit{linear} low-rank recovery methods for big datasets. Our method can also be seen as a major step towards efficient recovery of non-linear low-rank structures. For a matrix of size $n \times p$, on a single core machine, our method gains a speed up of $p^2/k$ over Robust Principal Component Analysis (RPCA),  where $k \ll p$ is the subspace dimension. Numerically, we can recover a low-rank matrix of size $10304 \times 1000$,  100 times faster than Robust PCA.
\end{abstract} 
\IEEEpeerreviewmaketitle

\section{\textbf{Introduction}}

In many applications in signal processing, computer vision and machine learning, the data has an intrinsic low-rank structure. One desires to extract this structure efficiently from the noisy observations. {Robust Principal Component Analysis (RPCA)} \cite{candes2011robust}, a linear dimensionality reduction algorithm can be used to exactly describe a dataset lying on a single linear low-dimensional subspace. Low-rank Representation (LRR) \cite{liu2013robust}, on the other hand can be used for data drawn from multiple linear subspaces.  However, these methods suffer from two prominent problems:
\begin{enumerate}
\item They do not recover non-linear low-rank structures.
\item They do not scale for big datasets $Y \in \Re^{p \times n}$ (large $p$ and large $n$, where $p$ is the number of features).
\end{enumerate}

{Many high dimensional datasets lie intrinsically on a smooth and very low-dimensional manifold that can be characterized by a graph $\G$ between the data samples \cite{belkin2003laplacian}.
For a matrix $Y \in \Re^{p \times n}$, a $\mathcal{K}$-nearest neighbor undirected graph between the rows or columns of $Y$ is denoted as $G = \mathcal{(V,E)}$, where $\mathcal{E}$ is the set of edges and $\mathcal{V}$ is the set of vertices.  The first step in the construction of $G$ consists of connecting each $y_i$ to its $\mathcal{K}$ nearest neighbors $y_j$ (using Euclidean distance), resulting in $|\mathcal{E}|$ connections. {The $y_i$ correspond to rows of $Y$ if the graph $G$ is the row graph or to the columns if $G$ is a column graph}. The $\mathcal{K}$-nearest neighbors are non-symmetric but a  symmetric weighted adjacency matrix $W$ is computed via a Gaussian kernel as $W_{ij} = \exp(- \|(y_i-y_j)\|^{2}_{2}/\sigma^{2})$ if $y_j$ is connected to $y_i$ {or vice versa} and 0 otherwise. Let $D$ be the diagonal degree matrix of $G$ which is given as: $D_{ii} = \sum_j W_{ij}$. Then, the combinatorial Laplacian that characterizes the graph $G$ is defined as $\Larg = D - W$ and its normalized form as $\Larg_n = D^{-1/2}(D-W)D^{-1/2}$ \cite{shuman2013emerging}. }

\begin{figure*}
\includegraphics[width=1.0\textwidth]{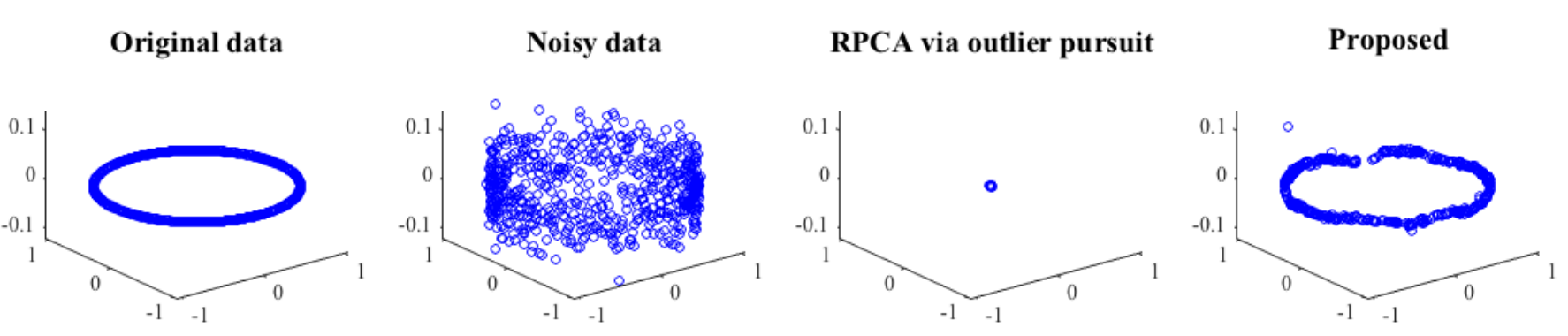}
\caption{A 2D circle and its noisy version embedded in a 3D space, which qualify as a non-linear low-rank structure. The goal is to recover the circle from noise as shown in  the rightmost plot, however, the state-of-the-art RPCA reduces the manifold to a point. Thus RPCA is not suitable to recover the non-linear low-rank structures.}
\label{fig:demo} 
\end{figure*}

{It is imperative to represent such datasets as a function of the smooth variations of the non-linear manifold, rather than a linear subspace. We refer to such a representation as a \textit{non-linear low-rank structure}. In this context, the graph eigenvectors serve as a legitimate tool to characterize the smooth variations of the manifold. Consider the example of a 2D circle embedded in a 3D space as shown in the left most plot of Fig. \ref{fig:demo}. The noisy version of this circle qualifies as an example of a non-linear low-rank (2D) manifold embedded in a high dimensional (3D) space. Ideally one would like to recover the 2D circle as shown in the rightmost plot of Fig. \ref{fig:demo}, however, RPCA just reduces the manifold to a point in the space. } Extensions of RPCA and LRR such as Robust PCA on Graphs (RPCAG) \cite{shahid2015robust} and Graph Regularized LRR (GLRR) \cite{lu2013graph} propose to incorporate graph regularization as a method to recover non-linear low-rank structures.  These methods still suffer from the scalability problem for big datasets. 

Randomized techniques come into play to deal with the scalability problem associated with very high dimensional data (the case of large $p$) \cite{tropp2008conditioning, boutsidis2009improved, witten2013randomized, li2015identifying, halko2011finding, oh2015fast, rahmani2015high, rahmani2015randomized, ha2015robust} using the tools of compression \cite{davenport2010signal}. These works improve upon the computational complexity by reducing only the \textit{data dimension $p$} but still scale in the same manner w.r.t $n$. The case of large $n$ can be tackled by using the sampling schemes accompanied with Nystrom method \cite{talwalkar2010matrix}. However, this method works efficiently only for low-rank kernel matrices and does not recover the low-rank data matrix itself.  Scalable extensions of LRR such as \cite{you2016scalable} exist but they focus only on the subspace clustering application. {Recently, Aravkin et. al \cite{aravkin2014variational} proposed to speed-up RPCA and ease the parameter selection problem, however, the variational approach does not qualify to represent the non-linear low-rank structures.} How to tackle the case of big $n$ and non-linearity simultaneously then?

%Factorized methods, which comprise the non-convex side of the picture  \cite{jiang2013graph, zhang2013low} are faster and scale well but this comes at the price of the loss of convexity and a prior knwoledge about the rank of the data. 

For many machine learning applications involving big data, such as clustering, an approximate low-rank representation might suffice.  The recently introduced Fast Robust PCA on Graphs (FRPCAG) \cite{shahid2015fast}  approximates a recovery method for non-linear low-rank datasets, which are called \textit{Low-rank matrices on graphs}. Inspired by the underlying stationarity assumption \cite{2016arXiv160102522P}, the authors introduce a joint notion of low-rankness for the features and samples (rows and columns) of a data matrix. More specifically, a low-rank matrix on graphs is defined as a matrix whose rows and columns belong to the span of the first few eigenvectors of the graphs constructed between its rows and columns. 
%As shown in \cite{shahid2015fast}, many real world data matrices can be considered to be low-rank on graphs    due to an underlying stationarity assumption . 

FRPCAG does not require an SVD and scales linearly with $n$. {It relies on fast dual graph filtering operations which involve matrix vector multiplications and can be parallelized on a GPU in every iteration. However, the size of the problem is still an issue for big datasets because the problem cannot be broken down into small sub-problems and the solution merged at the end. Thus,  for the non-GPU implementation, it  still suffers from 1) memory requirements  2) cost of k-means for clustering 3) the cost of parameter tuning for large $p$ and large $n$ and 4) scalability for very big datasets}.     This said, sometimes one might not even have access to the full dataset $Y$. This is typical, for instance for the biomedical applications, such as MRI and tomography. In such applications the number of observations are limited by the data acquisition protocols. In MRI, the number of observations is proportional to the  time and dose required for the procedure. In tomography one might have access to the projections only.   Thus, FRPCAG is  not be usable if 1) the dataset is large and 2) only a  subset of the dataset or measurements are available.  Despite the above limitations of the data acquisition, one might have access to some additional information about the unobserved samples. In MRI for instance, sparsity of the samples in the Fourier domain serves as a good prior. 

\subsection{\textbf{The Problem Statement}}
{In this work} we answer the following questions: \textit{1) What would be an efficient and highly scalable recovery framework, involving compression, for datasets which are jointly low-rank on two manifolds? 2) Alternatively, given a few randomly sampled observations and features from a data matrix $Y \in \Re^{p \times n}$, is it possible to efficiently recover the complete non-linear low-rank representation? } {We mostly limit ourselves to the case 1 above, where a graphical prior is available or can be conveniently constructed for the complete set of observations for the application under consideration}.  A brief initial treatment of the 2nd case constitutes Section VII.C of this work.

\subsection{\textbf{Contributions}}

PCA has been widely used for two different types of applications: 1) Low-rank recovery and 2) clustering in the low-dimensional space. It is crucial to point out here that the  clustering is not a standard application of PCA, because PCA is just a feature extraction method. However, the clustering experiments had been widely adopted as a standard procedure to demonstrate the quality of the feature extraction methods \cite{gao2013laplacian}, \cite{tao2014low}, \cite{zhang2013low}, \cite{jin2014low}, \cite{cai2011graph}, \cite{shang2012graph}, \cite{jiang2013graph}. Thus, to be consistent with the state-of-the-art, our contributions focus on both of the above applications. Below we describe our contributions in detail.

 \textbf{1. Sampling \& RIP for low-rank matrices on graphs:} To solve the scalability problem of FRPCAG we propose to perform a dual uniform sampling of the data matrices, along rows and columns. We present a restricted isometry property  (RIP) for low-rank matrices on graphs and relate it to the cumulative coherence of the graph eigenvectors.  FRPCAG is then used to  recover the low-rank representation for the sampled data. 
 
\textbf{2. Decoders for low-rank recovery:}  We present two (ideal and alternate) convex and efficient decoders for recovering the full low-rank matrix from the corresponding low-rank matrix of the sampled data. However, our main contribution comprises the set of 3 additional parallel, low-cost and parameter-free \textit{approximate} decoders, which significantly boost the speed of our framework by introducing a few approximations. Our rigorous theoretical analysis also proves that the recovery error of the above decoders depends on the spectral gaps of the row and column graph Laplacians.

\textbf{3. Low-Rank Clustering:} For the clustering application of PCA, we propose a low-cost and parallel scheme based on CPCA.   The key idea is to decode the labels of the complete dataset from the labels of a sampled low-rank dataset, without computing the complete low-rank matrix.

 \textbf{4. Extensive Experimentation:} Low-rank recovery experiments on 3 real video datasets and clustering experiments on 5 benchmark datasets reveal that the performance of our model is comparable to 10 different state-of-the-art PCA and non-PCA based methods. {We also study some cases where CPCA fails to perform as well as the state-of-the-art.}

Our proposed framework is inspired by the recently introduced sampling of band-limited signals on graphs \cite{puy2015random}. While we borrow several concepts from here, our framework is significantly different from \cite{puy2015random} in many contexts. We target the \textit{low-rank} recovery of matrices, whereas \cite{puy2015random} targets the recovery of band-limited signals / vectors. For our framework it is important for the data matrix to be low-rank jointly on the row and column graphs. Thus, our sampling scheme and RIP are generalized for two graphs. The design of a sampling scheme is the major focus of \cite{puy2015random}, while we just focus on the case of \textit{uniform sampling} and instead focus on \textit{how much to sample jointly} given the two graphs. Of course, our method can be extended directly for the other sampling schemes in  \cite{puy2015random}. A major difference lies in the application domain and hence the experiments. Unlike \cite{puy2015random}, we target two applications related to PCA: 1) low-rank recovery and 2) clustering. Thus, contrary to \cite{puy2015random} our proposed decoders are designed for these applications. A major contribution of our work in contrast to \cite{puy2015random} is the design of approximate decoders for low-rank recovery and clustering which significantly boost the speed of our framework for big datasets without compromising on the performance.

% \subsection{What this work is not about?} We specifically point out that this work is not about learning graphs. While the complexity of the fastest available graph construction algorithm FLANN \cite{muja2014scalable} is more than that of CPCA, fast graph construction is a separate area of research. Therefore, we do not consider the complexity of graph construction in the final complexity calculations. We  rely on the availability of graphs from external source.

\section{ \textbf{A Glimpse of Compressive PCA (CPCA)}}\label{sec:graphs}

% \subsection{Low-rank matrices on graphs}
   Let $\Larg_{c}\in \mathbb{R}^{n\times n}$ be the Laplacian of the graph $G_{c}$ connecting the different columns of $Y$  and $\Larg_{r} \in \mathbb{R}^{p\times p}$ the Laplacian of the graph $G_{r}$ that connects the rows of $Y$. Furthermore, let $\Larg_c = Q\Lambda_c Q^\top = Q_{k_c} \Lambda_{c k_c} Q^{\top}_{k_c} + \bar{Q}_{k_c} \bar{\Lambda}_{c k_c} \bar{Q}^{\top}_{k_c} $, where $\Lambda_{c k_c} \in \Re^{k_c \times k_c}$ is a diagonal matrix of lower eigenvalues  and    $\bar{\Lambda}_{c k_c} \in \Re^{(n - k_c) \times (n - k_c)}$ is  a diagonal matrix of higher graph eigenvalues. Similarly, let $\Larg_r = P\Lambda_r P^\top = P_{k_r} \Lambda_{r k_r} P^{\top}_{k_r} + \bar{P}_{k_r} \bar{\Lambda}_{r k_r} \bar{P}^{\top}_{k_r} $. All the values in $\Lambda_r$ and $\Lambda_c$ are sorted in increasing order. For a $\mathcal{K}$-nearest neighbors graph constructed from  $k_c$-clusterable data (along columns) one can expect $\lambda_{k_c}/\lambda_{k_c + 1} \approx 0$ as $\lambda_{k_c} \approx 0$ and $\lambda_{k_c} \ll \lambda_{k_c + 1}$. We refer to the ratio $\lambda_{k_c}/\lambda_{k_c + 1}$ as the spectral gap of $\Larg_c$. The same holds for the Laplacian $\Larg_r$. Then, low-rank matrices on graphs can be defined as following and recovered by solving FRPCAG \cite{shahid2015fast}.
%\vspace{-0.1cm}
\begin{defn}
{A matrix $Y^* \in \Rbb^{p \times n}$ is $(k_r, k_c)$-low-rank on the graphs $\Larg_r$ and $\Larg_c$ if its columns $y_j \in {\rm span}(P_{k_r})$ for all $j = 1, \ldots, n$ and its rows $y_i \in {\rm span}(Q_{k_c})$ for all $i = 1, \ldots, p$. The set of  $(k_r, k_c)$-low-rank matrices on the graphs $\Larg_r$ and $\Larg_c$ is denoted by $\mathcal{LR}( P_{k_r},Q_{k_c})$.}
\end{defn}

%More details on FRPCAG and low-rank matrices on graphs are presented in Section \ref{sec:frpcag}. The readers should also refer to \cite{shahid2015fast}.

%
% An approximate low-rank matrix $X$ for any $Y \in \mathcal{LR}( P_{k_r},Q_{k_c})$ can be determined via 
% % \begin{align}\label{eq:frpcag}
% % \min_{X} \|Y-X\|_{1} + \gamma_{c}\tr(X\Larg_{c} X^\top ) + \gamma_{r}\tr(X^\top \Larg_{r} X)
% % \end{align}
% %

%\subsection{Proposed scheme \& organization of the paper} 
%In the sequel we present the full compressive PCA framework for two different cases:
%\subsubsection{Case I: When $Y$ is available}
Given a data matrix $Y \in \Re^{p \times n} = \bar{X} + \bar{E}$, where $\bar{X} \in \mathcal{LR}( P_{k_r},Q_{k_c})$ and $\bar{E}$ models the errors, the goal is to develop a method to efficiently recover $\bar{X}$. We propose to 1) Construct Laplacians $\Larg_r$ and $\Larg_c$ between the rows and columns of $Y$ using the scheme of Section \ref{sec:graphs}. 2)  Sample the rows and columns of $Y$ to get a subsampled matrix $\tilde{Y} = \tilde{Y}^* + {E}$ using the sampling scheme of Section \ref{sec:RIP}. 3) Construct the compressed Laplacians $\tilde{\Larg}_r, \tilde{\Larg}_c$ from $\Larg_r,\Larg_c$ (Section \ref{sec:small_graphs}). 4) Determine a low-rank matrix $\tilde{X} $ for $\tilde{Y}$ with $\tilde{\Larg}_r, \tilde{\Larg}_c$ in algorithm 1 of FRPCAG:
\begin{align*}
 \min_{\tilde{X}} \phi(\tilde{Y}-\tilde{X}) + \gamma_{c}\tr(\tilde{X}\tilde{\Larg}_{c} \tilde{X}^\top ) + \gamma_{r}\tr(\tilde{X}^\top \tilde{\Larg}_{r} \tilde{X}),
 \end{align*}
 where $\phi$ is a loss function (possibly $l_p$ norm), $\tilde{X} = \tilde{X}^* + \tilde{E} = M_r \bar{X} M_c + \tilde{E}$, $\tilde{E}$ models the errors in the recovery of the subsampled low-rank matrix $\tilde{X}$ and $M_r, M_c$ are the row and column sampling matrices whose design is discussed in Section \ref{sec:RIP}. 5) Use the decoders presented in Section \ref{sec:decoders} to decode the low-rank matrix $\bar{X} = \bar{X}^* + E^*$  {(where $E^*$ denotes the error on the recovery of optimal $\bar{X}^*$)} on graphs  $\Larg_r, \Larg_c$ if the task is low-rank recovery, or perform k-means on $\tilde{X}$ to get cluster labels $\tilde{C}$ and use the clustering algorithm (presented in Section \ref{sec:clustering}) to get the cluster labels $C$ for the full matrix $X$.
 
 Throughout this work we use the approximate nearest neighbor algorithm (FLANN \cite{muja2014scalable}) for graph construction whose complexity is $\mathcal{O}(np \log (n))$ for $p \ll n$ \cite{sankaranarayanan2007fast} (and it can be performed in parallel).
%\end{enumerate}

\section{ \textbf{ RIP for low-rank matrices on graphs}}\label{sec:RIP}
%\subsection{The Uniform Sampling Matrices}
Let $M_r \in \Re^{\rho_r \times p }$ be the subsampling matrix for sampling  the rows and $M_c \in \Re^{n \times \rho_c }$    for sampling  the columns of $Y$. $M_c$ and $M_r$ are constructed by drawing $\rho_{c}$ and $\rho_r$ indices  $\Omega_c = \{\omega_1 \cdots \omega_{\rho_c}\}$ and $\Omega_r = \{\omega_1 \cdots \omega_{\rho_r}\}$ \textbf{uniformly without replacement} from the sets $\{1,2,\cdots, n\}$ and $\{1,2,\cdots, p\}$ and satisfy:
\begin{align}\label{eq:Mc}
M^{ij}_{c} = \left\{
\begin{array}{cc}
      1 & \text{if} \hspace{0.1cm} i = \omega_j \\
      0 & \text{otherwise}
\end{array} 
~
M^{ij}_{r} = \left\{
\begin{array}{cc}
      1 & \text{if} \hspace{0.1cm} j = \omega_i \\
      0 & \text{otherwise.}
\end{array} \right.
\right. 
\end{align} 
Now, the subsampled data matrix $\widetilde{Y} \in \Re^{\rho_c \times \rho_r}$ can be written as $\widetilde{Y} = M_r Y M_c $.
CPCA requires $M_r$ and $M_c$ to be constructed such that the ``low-rankness'' property of the data $Y$ is preserved under sampling.    Before discussing this, we introduce a few basic definitions in the context of graphs $G_c$ and $G_r$.
%

%\subsection{Graph Cumulative Coherence}
\begin{defn}\label{defn:cumulative-coherence}
(\textbf{Graph cumulative coherence}). {The cumulative coherence of order $k_c, k_r$ of $G_c$ and $G_r$ is:
$$
\nu_{{k_c}} = \max_{1 \leq i \leq n} \sqrt{n}\|Q^{\top}_{k_c}\Delta^c_{i}\|_{2} ~ ~ \text{\&} ~~ \nu_{{k_r}} = \max_{1 \leq j \leq p} \sqrt{p}\|P^{\top}_{k_r}\Delta^r_{j}\|_{2},
$$
where $\Delta^c \in \{0,1\}^n, \Delta^r \in \{0,1\}^p$ are binary vectors and $\Delta^c_i = 1$ if the $i^{th}$ entry of $\Delta^c$ is 1 and 0 otherwise. Thus, $\Delta^c_i$ corresponds to a specific node of the graph. }
 \end{defn}
{In the above equations $Q^{\top}_{k_c}\Delta^c_{i}$ and $P^{\top}_{k_r}\Delta^r_{j}$ characterize the first $k_c$ and $k_r$ fourier modes \cite{shuman2013emerging} of the nodes $i$ and $j$ on the graphs $G_c$ and $G_r$ respectively.} Thus, the cumulative coherence is a measure of how well the energy of the $(k_r, k_c)$ low-rank matrices  spreads over the nodes of the graphs.   These quantities exactly control the number of vertices $\rho_c$ and $\rho_r$ that need to be sampled from the graphs $G_r$ and $G_c$ such that the properties of the graphs are preserved \cite{puy2015random}.  
 
{ Consider the example where a particular node $i$ has a high coherence}. Then, it implies that their exist some low-rank signals whose energy is highly concentrated on the node $i$. Removing this node would result in a loss of information in the data. If the coherence of this node is low then removing it in the sampling process would result in no loss of information.  We already mentioned that we are interested in the case of uniform sampling. Therefore in order to be able to sample a small number of nodes uniformly from the graphs, the cumulative coherence should be as low as possible.

%\subsection{Implications of the Graph Cumulative Coherence}
We remind that for our application we desire to sample the data matrix  $Y$ such that its low-rank structure is preserved under this sampling. How can we ensure this via the graph cumulative coherence? This follows directly from  the fact that we are concerned about the data matrices which are also  low-rank with respect to the two graphs under consideration $ Y \in \mathcal{LR}( P_{k_r},Q_{k_c})$. In simple words, the columns of the data matrix $Y$ belong to the span of the eigenvectors $P_{k_r}$ and the rows to the span of $Q_{k_c}$. Thus, the coherence conditions for the graph  directly imply the coherence condition on the data matrix $Y$ itself. Therefore, using these quantities to sample the data matrix $Y$ will ensure the preservation of  two properties under sampling: 1) the structure of the corresponding graphs  and 2) the low-rankness of the data matrix $Y$. Given the above definitions, we are now ready to present the restricted-isometry theorem for the low-rank matrices on the graphs.

%As we are concerned about data matrices $Y \in \mathcal{LR}( P_{k_r},Q_{k_c})$, the coherence for the graphs  directly imply the coherence conditions on $Y$.  
%
%\subsection{RIP for Low-rank Matrices on Graphs}
\begin{thm}\label{thm:embedding} (\textbf{Restricted-isometry property  (RIP) for low-rank matrices on graphs})
Let $M_c$ and $M_r$ be two random subsampling matrices as constructed in \eqref{eq:Mc}. For any $\delta, \epsilon \in (0, 1)$, with probability at least $1 - \epsilon$,
\begin{align}\label{eq:rip}
(1-\delta) \|Y\|_F^2 \leq \frac{np}{\rho_r \rho_c} \|M_r Y M_c\|_F^2 \leq (1+\delta) \|Y\|_F^2
\end{align}
for all $Y \in \mathcal{LR}( P_{k_r},Q_{k_c})$ provided that
\begin{align}\label{eq:rhorc}
\rho_c \geq \frac{27}{\delta^2} \nu_{{k_c}}^2 \log\left(\frac{4k_c}{\epsilon}\right)
~\text{\&} ~
\rho_r \geq \frac{27}{\delta^2} \nu_{{k_r}}^2 \log\left(\frac{4k_r}{\epsilon}\right),
\end{align}
where $\nu_{k_c}, \nu_{k_r}$ characterize the graph cumulative coherence as in Definition \ref{defn:cumulative-coherence} and $\frac{np}{\rho_c \rho_r}$ is just a normalization constant which quantifies the norm conservation in~\eqref{eq:rip}.
\end{thm}
\begin{proof}
Please refer to Appendix~\ref{sec:embedding_proof}.
\end{proof}
 Theorem~\ref{thm:embedding} is a direct extension of the RIP for $k$-bandlimited signals on one graph \cite{puy2015random}. It states that the information in  $Y \in \mathcal{LR}( P_{k_r},Q_{k_c})$ is preserved with overwhelming probability if the sampling matrices~\eqref{eq:Mc} are constructed with a uniform sampling strategy satisfying~\eqref{eq:rhorc}. Note that $\rho_r$ and $\rho_c$ depend on the cumulative coherence of the graph eigenvectors. The better spread the eigenvectors are, the smaller is the number of vertices that need to be sampled. 
 
 It is proved in  \cite{puy2015random} that $\nu_{k_c} \geq \sqrt{k_c}$ and $\nu_{k_r} \geq \sqrt{k_r}$. Hence, when the lower bounds are attained, one only needs to sample an order of $O(k_c \log(k_c))$ columns and $O(k_r \log(k_r))$ rows to ensure that the RIP (eq. \eqref{eq:rhorc}) holds. This is the ideal scenario. However, one can also have $\nu_{k_c} = \sqrt{n}$ or $\nu_{k_r} = \sqrt{p}$ in some situations. Let us give some examples.

The lower bound on $\nu_{k}$ is attained, e.g, when the graph is the regular lattice. In this case the graph Fourier transform is the ``usual'' Fourier transform and $\nu_{k} = \sqrt{k}$ for all $k$. Another example where the lower bound is attained is when the graph contains $k$ disconnected components of identical size. In this case, one can prove that $\nu_{k} = \sqrt{k}$. Intuitively, we guess that the coherence remains close to this lower bound when these $k$ components are weakly interconnected.

The upper bound on $\nu_{k}$ is attained when, for example, the graph has one of its nodes not connected to any other node. In this case, one must sample this node. Indeed, there is no way to guess the value of the signal on this node from any neighbour. As the sampling is random, one is sure to sample this node only when all the nodes are sampled. Uniform sampling is not the best strategy in this setting. Furthermore, note that such a case is only possible if the graph is noisy or the data has strong outliers. One should resort to a more distribution aware sampling in such a case as presented in \cite{puy2015random}.

We choose in this paper to present the results using a uniform distribution for simplicity. Note however that one can adapt the sampling distribution to the underlying structure of the graph to ensure optimal sampling results. A consequence of the result in \cite{puy2015random} is that there always exist distributions that ensure that the RIP holds when sampling $O(k_r \log(k_r))$ rows and $O(k_c \log(k_c))$ columns only. The optimal sampling distribution for which this result holds is defined in \cite{puy2015random} (see Section 2.2). Furthermore, a fast algorithm to compute this distribution also exists (Section 4 of \cite{puy2015random}). 

{\section{\textbf{Compressed Low-Rank Matrix}}
Once the compressed dataset $\tilde{Y} \in \Re^{\rho_r \times \rho_c}$ is obtained the low-rank representation has to be extracted which takes into account the graph structures. Thus we propose the following two step strategy:
\begin{enumerate}
\item Construct graphs for compressed data.
\item Run Fast Robust PCA on Graphs (FRPCAG) on the compressed data.
\end{enumerate}
These two steps are elaborated in the following subsections.}

\subsection{\textbf{Graphs for Compressed data}}\label{sec:small_graphs}
To ensure the preservation of algebraic and spectral properties one can construct the compressed Laplacians $\tilde{\Larg}_r \in \Re^{\rho_r \times \rho_r}$ and $\tilde{\Larg}_c \in \Re^{\rho_c \times \rho_c}$ from the Kron reduction of $\Larg_r$ and $\Larg_c$ \cite{dorfler2013kron}. Let $\Omega$ be  the set of sampled nodes and $\bar{\Omega}$ the complement set and let $\Larg(A_r, A_c)$ denote the (row, column) sampling of $\Larg$ w.r.t sets $A_r, A_c$ then the Laplacian $\tilde{\Larg}_c$ for the columns of compressed matrix $\tilde{Y}$ is:
$$\tilde{\Larg}_c = \Larg_c (\Omega, \Omega) - \Larg_c (\Omega, \bar{\Omega}) \Larg^{-1}_c (\bar{\Omega},\bar{\Omega}) \Larg_c (\bar{\Omega}, \Omega).$$
Let $\Larg_c$ has $k_c$ connected components or $\lambda_{k_c}/\lambda_{k_c + 1} \approx 0$. Then, as argued in theorem $III.4$ of \cite{dorfler2013kron} two nodes $\alpha, \beta$ are not connected in $\tilde{\Larg}_c$ if there is no path between them in $\Larg_c$ via $\bar{\Omega}$. {Assume that each of the connected components has the same number of nodes. Then, if the sampling is done uniformly within each of the connected components according to the sampling bounds described in eq.\eqref{eq:rhorc}, one can expect $\tilde{\Larg}_c$ to have $k_c$ connected components as well. This is an inherent property of the Kron reduction method. However, for the case of large variation of the number of nodes among the connected components one might want to resort to a more distribution aware sampling scheme. Such schemes have been discussed in \cite{puy2015random} and have not been addressed in this work. Nevertheless, the Kron reduction strategy mentioned here is independent of the sampling strategy used}. The same concepts holds for $\tilde{\Larg}_r$ as well. 

The Kron reduction method involves the multiplication of 3 sparse matrices. The only expensive operation above is the inverse of $\Larg(\bar{\Omega}, \bar{\Omega})$ which can be performed with $\mathcal{O}(O_l \mathcal{K} n)$ cost using the Lancoz method \cite{susnjara2015accelerated}, where $O_l$ is the number of iterations for Lancoz approximation.

\subsection{\textbf{FRPCAG on the Compressed Data}}\label{sec:frpcag}
 Once the Laplacians $\tilde{\Larg}_r \in \Re^{\rho_r \times \rho_r}, \tilde{\Larg}_c \in \Re^{\rho_c \times \rho_c}$ are obtained, the next step is to recover the low-rank matrix $\tilde{X} \in \Re^{\rho_r \times \rho_c}$. Let $\tilde{\Larg}_c = \tilde{Q}\tilde{\Lambda}_c \tilde{Q}^\top = \tilde{Q}_{k_c} \tilde{\Lambda}_{c k_c} \tilde{Q}^{\top}_{k_c} + \bar{\tilde{Q}}_{k_c} \bar{\tilde{\Lambda}}_{c k_c} \bar{\tilde{Q}}^{\top}_{k_c} $, where $\tilde{\Lambda}_{k_c} \in \Re^{k_c \times k_c}$ is a diagonal matrix of lower eigenvalues  and    $\bar{\tilde{\Lambda}}_{k_c} \in \Re^{(\rho_c - k_c) \times (\rho_c - k_c)}$ is  a diagonal matrix of higher graph eigenvalues. Similarly, let $\tilde{\Larg}_r = \tilde{P}\tilde{\Lambda}_r \tilde{P}^\top = \tilde{P}_{k_r} \tilde{\Lambda}_{r k_r} \tilde{P}^{\top}_{k_r} + \bar{\tilde{P}}_{k_r} \bar{\tilde{\Lambda}}_{r k_r} \bar{\tilde{P}}^{\top}_{k_r} $. Furthermore assume that all the values in $\tilde{\Lambda}_r$ and $\tilde{\Lambda}_c$ are sorted in increasing order.

Assume $\tilde{Y} = \tilde{Y}^* + {E}$, where ${E}$ models the noise in the compressed data and $ \tilde{Y}^* \in \mathcal{LR}( \tilde{P}_{k_r},\tilde{Q}_{k_c})$. The low-rank matrix $\tilde{X} = \tilde{X}^* + \tilde{E}$ can be recovered by solving the FRPCAG problem as proposed in \cite{shahid2015fast} and re-written below:
\begin{align}\label{eq:frpcag_small}
 \min_{\tilde{X}} \phi(\tilde{Y}-\tilde{X}) + \gamma_{c}\tr(\tilde{X}\tilde{\Larg}_{c} \tilde{X}^\top ) + \gamma_{r}\tr(\tilde{X}^\top \tilde{\Larg}_{r} \tilde{X}),
 \end{align}
where $\phi$ is a proper, positive, convex and lower semi-continuous loss function (possibly $l_p$ norm). From Theorem 1 in \cite{shahid2015fast}, the low-rank approximation error comprises the orthogonal projection of $\tilde{X}^*$ on the complement graph eigenvectors ($\bar{\tilde{Q}}_{k_c}, \bar{\tilde{P}}_{k_r}$) and depends on the spectral gaps $\tilde{\lambda}_{k_c}/\tilde{\lambda}_{k_c + 1}, \tilde{\lambda}_{k_r}/\tilde{\lambda}_{k_r + 1}$ as following:
\begin{align}\label{eq:tildeE}
&\| \tilde{X}^{*} \bar{\tilde{Q}}_{k_c}\|_F^2 + \|\bar{\tilde{P}}_{k_r}^\top  \tilde{X}^*\|_F^2 =  \|\tilde{E}\|^2_F \nonumber \\
& \leq \frac{1}{\gamma}\phi({E}) +  \|\tilde{Y}^*\|_F^2 \Big( \frac{\tilde{\lambda}_{k_c}}{\tilde{\lambda}_{k_c+1}} + \frac{\tilde{\lambda}_{k_r}}{\tilde{\lambda}_{k_r+1}} \Big),
\end{align}
where $\gamma$ depends on the signal-to-noise ratio. Clearly, if ${\lambda}_{k_c}/{\lambda}_{k_c + 1} \approx 0$ and ${\lambda}_{k_r}/{\lambda}_{k_r + 1} \approx 0$ and the compressed Laplacians are constructed using the Kron reduction then   $\tilde{\lambda}_{k_c}/\tilde{\lambda}_{k_c + 1} \approx 0$ and $\tilde{\lambda}_{k_r}/\tilde{\lambda}_{k_r + 1} \approx 0$. Thus, exact recovery is attained. 

Let $g(Z) = \gamma_c \tr(Z \Larg_c Z^\top) + \gamma_r \tr(Z^\top \Larg_r Z)$, then $\nabla_g(Z) =  2( \gamma_{c} Z\Larg_{c}  + \gamma_{r} \Larg_{r}Z)$. Also define $\prox_{\lambda h }(Z) = Y + \sign(Z-Y) \circ \max (|Z-Y|-\lambda ,0)$, where $\circ$ denotes the Hadamard product, $\lambda$ as the step size (we use $\lambda = \frac{1}{\beta '}$), where $\beta \leq \beta' = 2\gamma_c \|\Larg_c\|_2 + 2\gamma_r \|\Larg_r\|_2$ and $\|\Larg \|_2$ is the  spectral norm (or maximum eigenvalue) of $\Larg$, $\epsilon$ as the stopping tolerance and $J$ the maximum number of iterations.  Then FRPCAG can be solved by the FISTA in Algorithm 1.  

\begin{algorithm}
\caption{FISTA for FRPCAG}
\label{algorithm1}
\begin{algorithmic}
\State INPUT: $Z_1 = Y$, $S_0 = Y$, $t_1 = 1$, $\epsilon > 0$
\For{ $j = 1,\dots J$ }
\State $S_{j} = \prox_{\lambda_{j}h}(Z_{j}-\lambda_{j}\nabla g(Z_{j}))$
\State $t_{j+1} = \frac{1+\sqrt{1+4t_j^2}}{2}$
\State $Z_{j+1} = S_j +\frac{t_j-1}{t_{j+1}} (S_j-S_{j-1})$
\If{$\|Z_{j+1} - Z_{j}\|_F^2 < \epsilon \| Z_{j}\|_F^2$}
\State BREAK
\EndIf
\EndFor
\State OUTPUT: $U_{j+1}$
%\EndProcedure
\end{algorithmic}
\end{algorithm}

\section{\textbf{Decoders for low-rank recovery}}\label{sec:decoders}
Let $\tilde{X} \in \Re^{\rho_r \times \rho_c}$ be the low-rank solution of \eqref{eq:frpcag_small} with the compressed graph Laplacians $\tilde{\Larg}_r, \tilde{\Larg}_c$ and sampled data $\tilde{Y}$. The goal is to decode the low-rank matrix $X \in \Re^{p\times n}$ for the full $Y$. We assume that $\tilde{X} = M_r \bar{X} M_c + \tilde{E}$, where $\tilde{E} \in \Re^{\rho_r \times \rho_c}$ models the noise incurred by \eqref{eq:frpcag_small}.
%The only information available for the decode stage is: 1) the Laplacians $\Larg_r$ and $\Larg_c$ of the matrix $Y$ and 2) the data matrix $Y \in  \mathcal{LR}( P_{k_r},Q_{k_c})$.

\subsection{\textbf{Ideal Decoder}}
A straight-forward way to  decode $X$ on the original graphs $\Larg_r$ and $\Larg_c$, when one knows the basis $P_{k_r}, Q_{k_c}$ involves solving the following optimization problem:
\begin{align}\label{eq:idealdecoder}
& \min_{X} \|M_r X M_c - \tilde{X}\|^{2}_{F} \nonumber \\
& \text{s.t:}  \hspace{0.2cm} (X)_i \in span(P_{k_r}), \hspace{0.1cm} (X^{\top})_j \in span(Q_{k_c}).
\end{align}
%
%Theorem~\ref{thm:idealdecoder} below establishes an upper error bound for the recovery of the complete low-rank matrix using this ideal decoder.
\begin{thm}\label{thm:idealdecoder}
 Let $M_r$ and $M_c$ be such that \eqref{eq:rip} holds and $X^*$ be the solution of \eqref{eq:idealdecoder} with $\tilde{X} = M_r \bar{X} M_c + \tilde{E}$, where $\bar{X} \in \mathcal{LR}(P_{k_r}, Q_{k_c})$ and $\tilde{E} \in \Re^{\rho_r \times \rho_c}$. We have:
\begin{equation}\label{eq:solidealdecoder}
\|X^*-\bar{X}\|_{F} \leq 2 \sqrt{\frac{np}{\rho_c\rho_r(1-\delta)}}\|\tilde{E}\|_{F},
\end{equation}
where $\sqrt{{np}/{\rho_c\rho_r(1-\delta)}}$ is a constant resulting from the norm preservation in~\eqref{eq:rip} and $\|\tilde{E}\|^2_F$ is bounded by eq. \eqref{eq:tildeE}.
\end{thm}
\begin{proof}
Please refer to Appendix~\ref{sec:idealdecoder}.
\end{proof}

Thus, the error of the ideal decoder is only bounded by the error $\tilde{E}$ in the low-rank matrix $\tilde{X}$  obtained by solving~\eqref{eq:frpcag_small}.  In fact $\tilde{E}$ depends on the spectral gaps of $\tilde{\Larg}_c, \tilde{\Larg}_r$, as given in eq. \eqref{eq:tildeE}. Hence, the ideal decoder itself does not introduce any error in the decode stage. The solution for this decoder requires projecting over the eigenvectors $P$ and $Q$ of $\Larg_r$ and $\Larg_c$. This is computationally expensive because diagonalization of $\Larg_r$ and $\Larg_c$ cost $\mathcal{O}(p^3)$ and $\mathcal{O}(n^3)$. Moreover, the constants $k_r, k_c$ are not known beforehand and require tuning.

\subsection{\textbf{Alternate Decoder}}
 As the ideal decoder is computationally costly, we propose to decode $X$ from $\tilde{X}$ by using a convex and computationally tractable problem which involves the minimization of graph dirichlet energies. %\vspace{-0.3cm}
\begin{equation}\label{eq:alternatedecoder}
\min_{X} \|M_r X M_c - \tilde{X}\|^{2}_{F} + \bar{\gamma_c}\tr(X\Larg_c X^{\top}) + \bar{\gamma_r}\tr(X^{\top}\Larg_r X).
\end{equation}
%
%Theorem~\ref{thm:alternatedecoder} below establishes an upper error bound for the recovery of the complete low-rank matrix using this alternate decoder.
\begin{thm}\label{thm:alternatedecoder}
Let $M_r$ and $M_c$ be such that \eqref{eq:rip} holds and $\gamma>0$. Let also $X^*$ be the solution of \eqref{eq:alternatedecoder} with $\bar{\gamma}_c = \gamma/\lambda_{k_c+1}$, $\bar{\gamma}_r = \gamma/\lambda_{k_r+1}$, and $\tilde{X} = M_r \bar{X} M_c + \tilde{E}$, where $\bar{X} \in \mathcal{LR}(P_{k_r}, Q_{k_c})$ and $\tilde{E} \in \Re^{\rho_r \times \rho_c}$. We have:
\begin{align}\label{eq:solalternatedecoder}
& \|\bar{X}^{*}-\bar{X}\|_{F} \leq \sqrt{\frac{np}{\rho_c\rho_r(1-\delta)}}\Bigg[\Big(2 + \frac{1}{\sqrt{2\gamma}}\Big) \|\tilde{E}\|_F + \nonumber \\
& (\frac{1}{\sqrt{2}} + \sqrt{\gamma})\sqrt{\Big(\frac{\lambda_{k_c}}{\lambda_{k_c + 1}} +  \frac{\lambda_{k_r}}{\lambda_{k_r + 1}}\Big)}\|\bar{X}\|_F\Bigg], 
\quad \text{and} \nonumber \\
& \|E^*\|_F \leq \frac{\|\tilde{E}\|_{F}}{\sqrt{2\gamma}} + \frac{1}{\sqrt{2}} \sqrt{\Big(\frac{\lambda_{k_c}}{\lambda_{k_c + 1}} +  \frac{\lambda_{k_r}}{\lambda_{k_r + 1}}\Big)}\|\bar{X}\|_{F},
\end{align}
where $\bar{X}^* = {\rm Proj}_{\mathcal{LR}(P_{k_r}, Q_{k_c})}(X)$ and $E^* = X^* - \bar{X}^*$. ${\rm Proj}_{\mathcal{LR}(P_{k_r}, Q_{k_c})}(.)$ denotes the orthogonal projection onto ${\mathcal{LR}(P_{k_r}, Q_{k_c})}$ and $\gamma$ depends on the signal to noise ratio.
\end{thm}
\begin{proof}
Please refer to Appendix~\ref{sec:alternatedecoder}
\end{proof}
Theorem~\ref{thm:alternatedecoder} states that in addition to the error $\tilde{E}$ in $\tilde{X}$ incurred by \eqref{eq:frpcag_small} and characterized by the bound in eq. \eqref{eq:tildeE}, the error of the alternate decoder \eqref{eq:alternatedecoder} also depends on the spectral gaps of the Laplacians $\Larg_r$ and $\Larg_c$ respectively. This is the price that one has to pay in order to avoid the expensive ideal decoder. For a $k_r, k_c$ clusterable data $Y$ across the rows and columns, one can expect $\lambda_{k_r}/\lambda_{k_r + 1} \approx 0$ and $\lambda_{k_c}/\lambda_{k_c + 1} \approx 0$ and the solution is as good as the ideal decoder. Nevertheless, it is possible to reduce this error by using graph filters $g$ such that the ratios ${g(\lambda_{k_c})}/{g(\lambda_{k_{c}+1})}$ and ${g(\lambda_{k_r})}/{g(\lambda_{k_{r}+1})}$ approach zero. However, we do not discuss this approach in our work. It is trivial to solve \eqref{eq:alternatedecoder} using a conjugate gradient  scheme that costs $\mathcal{O}(I n p \mathcal{K})$, where $I$ is the number of iterations for the algorithm to converge. 

\subsection{\textbf{Approximate Decoder}}\label{sec:approx_decoder}

The alternate decoder proposed above has the following disadvantages: 1) It is almost as computationally expensive as FRPCAG 2) It requires tuning two model parameters. 

In this section we describe the step-by-step construction of  an approximate decoder which overcomes these limitations. The main idea is to breakdown the decode phase  of low-rank matrix $X$ into its left and right singular vectors or subspaces. Let $X = U\Sigma V^\top$ and $\tilde{X }= \tilde{U}\tilde{\Sigma} \tilde{V}^\top$ be the SVD of $X$ and $\tilde{X}$. {We propose to recover $U$ from $\tilde{U}$ and $V$ from $\tilde{V}$ in 3 steps.
\begin{enumerate}
\item Split the alternate decoder to subspace learning problems.
\item Drop the orthonormality constraints on subspaces.
\item Run an efficient upsampling algorithm to solve the problem of step 2.
\end{enumerate}
The goal of this step-by-step approach is to guide the reader throughout to observe the close relationship between the alternate and approximate decoder. Now we begin to describe these steps in detail.}

\subsubsection{\textbf{Step 1: Splitting the alternate decoder}}
Using the SVD of $X$ and $\tilde{X}$ and the invariance property of the trace under cyclic permutations, we can replace \eqref{eq:alternatedecoder} by:
\begin{align}\label{eq:midstep}
\min_{U,V} & \|M_r U \Sigma V^{\top}M_c- \tilde{U}\tilde{\Sigma} \tilde{V}^\top\|^{2}_{F} + \bar{\gamma}_c\tr(\Sigma^{2} V^{\top}\Larg_c V) + \nonumber \\ & \bar{\gamma}_r\tr(U^{\top} \Larg_r U\Sigma^{2}) \quad \text{s.t:} \hspace{0.2cm} U^{\top}U = I_k, \hspace{0.1cm} V^{\top}V = I_k.
\end{align}

{The above eq. introduces two new variables based on the SVD of $X$, i.e, $U \in \Rbb^{p \times k}$ and $V \in \Rbb^{n \times k}$. Clearly, with the introduction of these new variables, one needs to specify $k$ as the dimension of the subspaces $U$ and $V$. We propose the following strategy for this: 
\begin{enumerate}
\item First, determine $\tilde{\Sigma}$ by one inexpensive SVD of $\widetilde{X} \in \Rbb^{\rho_r \times \rho_c}$. This costs $\mathcal{O}(\rho^2_r \rho_c)$ for $\rho_r < \rho_c$. 
\item Then set $k$ equal to the number of entries in $\tilde{\Sigma}$ which are above a threshold.
\end{enumerate}}

{
It is important to note that so far eq.\eqref{eq:midstep} and the alternate decoder eq.\eqref{eq:alternatedecoder} are equivalent. Also note that we did not introduce the singular values $\Sigma$ as an additional variable in eq.\eqref{eq:midstep} because they are related to the singular values $\tilde{X}$ of $\tilde{X}$. We argue this as following:  If \eqref{eq:solalternatedecoder} holds for the alternate decoder then $\|\bar{\Sigma}^* - \bar{\Sigma}\|_F$ (where $\bar{\Sigma}^*, \bar{\Sigma}$ are the singular values of $\bar{X}^*, \bar{X}$) is also bounded as argued in the discussion of Appendix \ref{sec:alternatedecoder}. Thus, the  singular values ${\Sigma}$ and $\tilde{\Sigma}$  of $X$ and $\tilde{X}$ differ approximately by the normalization constant of theorem \ref{thm:embedding}, i.e,
$$\Sigma = \sqrt{\frac{np}{\rho_r \rho_c (1-\delta)}}\tilde{\Sigma}$$
}

{Note that with the above relationship, the subspaces $U, V$ can be solved independently of each other. Thus eq.\eqref{eq:midstep} can be decoupled as following which separately solves the subspace ($U$ and $V$) learning problems. %\vspace{-0.2cm}
\begin{align}\label{eq:UUI}
\min_{U} & \|M_r U - \tilde{U}\|^{2}_{F} + {\gamma^{'}}_r\tr(U^{\top} \Larg_r U) \quad \text{s.t:} \quad U^{\top}U = I_k, \nonumber \\
\min_{V} & \|V^{\top}M_c - \tilde{V}\|^{2}_{F} + {\gamma^{'}}_c\tr(V^{\top} \Larg_c V) \quad \text{s.t:} \quad V^{\top}V = I_k.
\end{align} }

 \subsubsection{\textbf{Step 2: Dropping Orthonormality Constraints}}
 Solving \eqref{eq:UUI} is as expensive as \eqref{eq:alternatedecoder} due to the orthonormality constraints (as explained in appendix~\ref{sec:solUUI}). Therefore, we drop the constraints and get %\vspace{-0.4cm}
\begin{align}\label{eq:u}
\min_{U} & \|M_r U - \tilde{U}\|^{2}_{F} + {\gamma^{'}}_r\tr(U^{\top} \Larg_r U),
\end{align} %\vspace{-0.7cm}
\begin{align}\label{eq:v}
\min_{V} & \|V^{\top}M_c - \tilde{V}\|^{2}_{F} + {\gamma^{'}}_c\tr(V^{\top} \Larg_c V).
\end{align}

{The solutions to \eqref{eq:u} \& \eqref{eq:v} are not orthonormal anymore.  The deviation from the orthonormality depends on the constants $\gamma^{'}_r$ and $\gamma^{'}_c$, but $X = U\Sigma V^\top$ is still a good enough (error characterized in Theorem \ref{thm:approxdecoder}) low-rank representation due to the intuitive explanation that we present here. We argue that the solutions of eqs.\eqref{eq:u} \&\eqref{eq:v} are feasible solutions of the joint non-convex, factorized, and graph regularized low-rank optimization problem like the one presented in \cite{rao2015collaborative}. Let $A$ and $B$ be the subspaces that we want to recover then we can re-write the problem studied in \cite{rao2015collaborative} as following:
\begin{equation*}
\min_{A,B} \|M_rAB^\top M^\top_c - \tilde{X}\|^2_F + \gamma^{'}_r \tr(A^\top \Larg_r A) +   \gamma^{'}_c\tr(B^\top \Larg_c B)
\end{equation*}
The above non-convex problem does not require $A$ and $B$ to be orthonormal, but is still widely used for recovering a low-rank $X= AB^\top$.  Our problem setting (eqs.\eqref{eq:u} \&\eqref{eq:v}) is just equivalent except that it is convex as we decouple the learning of two subspaces due to the known $\Sigma$ that relates $U$ and $V$.  Thus, for any orthonormal $U, V$ and a scaling matrix $\Sigma$, $A= U\sqrt{\Sigma}$ and $B = V\sqrt{\Sigma}$ is a feasible solution. Thus, dropping the  orthonormality constraints does not effect the final solution $X$.}

\subsubsection{\textbf{Step 3: Subspace Upsampling}}%\vspace{-0.2cm}
Eqs. \eqref{eq:u} \&\eqref{eq:v} require the tuning of two parameters $\gamma^{'}_r$ and $\gamma^{'}_c$ which can be computationally cumbersome. Therefore, our final step in the construction of the approximate decoder is to get rid of the two parameters. But before we present the final construction step we study the problems eqs.\eqref{eq:u} \&\eqref{eq:v} and their solutions more closely. 

{First, note that solving eqs.\eqref{eq:u} \&\eqref{eq:v} is equivalent to making the following assumptions:}

$$\tilde{U} = M_r \bar{U} + \tilde{E}^u \quad \text{and} \quad \tilde{V} =  \bar{V}M_c + \tilde{E}^v,$$
{where the columns of $\bar{U}$, $\bar{u}_i \in span(P_{k_r})$, $i = 1, \cdots, p$, and the columns of $\bar{V}$ $\bar{v}_j \in span (Q_{k_c})$, $j = 1, \cdots, n$ and $\tilde{E}^u \in \Re^{\rho_r \times \rho_r}$, $\tilde{E}^v \in \Re^{\rho_c \times \rho_c}$ model the noise in the estimate of the subspaces.} 

Secondly, the closed form solutions of eqs.\eqref{eq:u} \&\eqref{eq:v} are given as following:
\begin{equation}
U = (M^{\top}_r M_r + \gamma^{'}_r \Larg_r)^{-1}M^{\top}_r \tilde{U},
\end{equation}%\vspace{-0.5cm}
\begin{equation}
V = (M_c M^{\top}_c + \gamma^{'}_c \Larg_c)^{-1}M_c \tilde{V}.
\end{equation}
Thus, problems \eqref{eq:u} \& \eqref{eq:v} decode the subspaces $U$ and $V$ such that they  are smooth on their respective graphs $\Larg_r$ and $\Larg_c$. \textit{This can also be referred to as 1) simultaneous decoding and 2) subspace denoising stage}. We call it a `subspace denoising' method because the operator $(M^{\top}_r M_r + \gamma^{'}_r \Larg_r)^{-1}$ can be viewed as low-pass filtering the subspace $U$ in the graph fourier domain.

Note that we want to decode and denoise $\tilde{U}$ and $\tilde{V}$ which are in turn determined by the SVD of $\tilde{X}$. Furthermore, $\tilde{X}$ has been determined by solving the FRPCAG problem of eq.\eqref{eq:frpcag_small}. FRPCAG is already robust to noise and outliers, therefore, it is safe to assume that the subspaces determined from it, i.e, $\tilde{U}$ and $\tilde{V}$ are also noise and outlier free. Thus, the extra denoising step (performed via graph filtering) of eqs.\eqref{eq:u} \&\eqref{eq:v} is redundant.

 Therefore, we can directly upsample $\tilde{U}$ and $\tilde{V}$ to determine $U$ and $V$ without needing a margin for noise. To do this, we reformulate eq.\eqref{eq:u} as follows:
 $$\min_{U}  \frac{1}{{\gamma^{'}}_r}\|M_r U - \tilde{U}\|^{2}_{F} + \tr(U^{\top} \Larg_r U).$$
 
 For $\gamma^{'}_r \rightarrow 0$, $\frac{1}{\gamma^{'}_r } \rightarrow \infty$, the emphasis on first term of the objective increases and it turns to an equality constraint $M_r U = \tilde{U}$. The same holds for eq.\eqref{eq:v} as well. Thus, the modified problems are:
 
\begin{align}\label{eq:approx_uv}
& \min_{U} \tr(U^{\top}\Larg_r U) \quad \text{and} \quad \min_{V} \tr(V^{\top}\Larg_c V) \nonumber \\
& \text{s.t:} \hspace{0.1cm} M_r U = \tilde{U}, \quad \quad \quad \quad \text{s.t:} \hspace{0.1cm} M^{\top}_c V = \tilde{V}.
\end{align}
Note that now we have a parameter-free decode stage.

It is important now to study the theoretical guarantees on eq. \eqref{eq:approx_uv}. To do this, as eqs. \eqref{eq:approx_uv} are a specific case of eqs. \eqref{eq:u} \& \eqref{eq:v}, we first study the guarantees on  eqs. \eqref{eq:u} \& \eqref{eq:v} in Theorem \ref{thm:approxdecoder}. Then, based on this study we directly present the guarantees on the \textit{final approximate decoder} of eq. \eqref{eq:approx_uv} in Theorem \ref{thm:approxdecoder2}.

\begin{thm}\label{thm:approxdecoder}
Let $M_r$ and $M_c$ be such that \eqref{eq:rip} holds and $\gamma_r', \gamma_c'>0$. Let also $U^*$ and $V^*$ be respectively the solutions of \eqref{eq:u} and \eqref{eq:v} with $\tilde{U} = M_r \bar{U} + \tilde{E}^u$ and $\tilde{V} = M_c  \bar{V} + \tilde{E}^v$, where $\bar{u}_i \in span(P_{k_r})$, $i = 1, \cdots, p$, $\bar{v}_j \in span (Q_{k_c})$, $j = 1, \cdots, n$, $\tilde{E}^u \in \Re^{\rho_r \times \rho_r}$, $\tilde{E}^v \in \Re^{\rho_c \times \rho_c}$. We have:
\begin{align*}
\norm{\bar{U}^* - \bar{U}}_F &
\leq 
\sqrt{\frac{2 p}{\rho_r(1-\delta)}}
\Bigg[ \left( 2 + \frac{1}{\sqrt{\gamma_r' \lambda_{k_r+1}}} \right) \norm{\tilde{E}^u}_F \\
& + \left( \sqrt{\frac{\lambda_{k_r}}{\lambda_{k_r+1}}} + \sqrt{\gamma_r' \lambda_{k_r}} \right) \norm{\bar{U}}_F \Bigg], ~ \text{and}
\end{align*}
\begin{align*}
\norm{E^*}_F
\leq 
\sqrt{\frac{2}{\gamma_r' \lambda_{k_r+1}}} \norm{\tilde{E}^u}_F
+ \sqrt{2\frac{\lambda_{k_r}}{\lambda_{k_r+1}}} \norm{\bar{U}}_F.
\end{align*}
where $\bar{U}^* = P_{k_r} P_{k_r}^\top X$ and $E^* = U^* - \bar{U}^*$. The same inequalities with slight modification also hold for $V^*$, which we omit because of space constraints.
\end{thm} 
\begin{proof}
Please refer to Appendix~\ref{sec:approxdecoder}.
\end{proof}%\vspace{-0.4cm}

\begin{thm}\label{thm:approxdecoder2}
Let $M_r$ and $M_c$ be such that \eqref{eq:rip} holds. Let also $U^*$ and $V^*$ be the solutions of \eqref{eq:approx_uv}  with $\tilde{U} = M_r {\bar{U}}$ and $\tilde{V} = M_c {\bar{V}}$, where $\bar{u}_i \in span(P_{k_r})$, $i = 1, \cdots, p$, $\bar{v}_j \in span (Q_{k_c})$, $j = 1, \cdots, n$. We have:
\begin{align*}
\norm{{U}^* - \bar{U}}_F &
\leq 
\sqrt{\frac{2 p}{\rho_r(1-\delta)}} \sqrt{\frac{\lambda_{k_r}}{\lambda_{k_r+1}}} \norm{\bar{U}}_F 
\end{align*}
where ${U}^* = P_{k_r} P_{k_r}^\top X$. The same inequalities with slight modification also hold for $V^*$, which we omit because of space constraints.
\end{thm} 
\begin{proof}
The proof directly follows from the proof of Theorem \ref{thm:approxdecoder} by using $\tilde{E}^u = 0$ and $\gamma^{'}_r = 0$.
\end{proof}%\vspace{-0.4cm}

As $\bar{X} = \bar{U} \bar{\Sigma} \bar{V}^{\top}$, we can say that the error with eqs.\eqref{eq:approx_uv}  is upper bounded by the product of the errors of the individual subspace decoders.  Also note that the error again depends on the spectral gaps  defined by the ratios $\lambda_{k_c}/\lambda_{k_{c}+1}$ and 
$\lambda_{k_r}/\lambda_{k_{r}+1}$. 
%\vspace{-0.3cm}

%\subsection{Solution of the Subspace Upsampling  Decoder}
The solution to the above problems is simply a graph upsampling operation as explained in Lemma~\ref{lemma:graphupsampling}.  

\begin{lemma}\label{lemma:graphupsampling}
Let $S \in \Re^{c \times r}$ and $R \in \Re^{d \times r}$ be the two matrices such that $d < r$ and $d < c$. Furthermore, let $M \in \Re^{d \times c}$ be a sampling matrix as constructed in~\eqref{eq:Mc} and $\Larg \in \Re^{c \times c}$ be a symmetric positive semi-definite matrix. We can write $S = [{S}^{\top}_{a} | S^{\top}_{b}]^{\top}$, where $S_b \in \Re^{d \times r}$ and $S_a \in \Re^{(c-d) \times r}$ are the known and unknown submatrices of $S$. Then the exact and unique solution to the following problem:
\begin{align}\label{eq:graphupsampling}
& \min_{S_a} \tr(S^{\top}\Larg S), \quad \text{s.t:} \quad M S = R
\end{align}
is given by $S_a = -\Larg^{-1}_{aa}\Larg_{ab}R $.
\end{lemma}%\vspace{-0.3cm}
\begin{proof}
Please refer to Appendix~\ref{sec:graphupsampling}.
\end{proof}%\vspace{-0.3cm}
Using Lemma~\ref{lemma:graphupsampling} and the notation of Section \ref{sec:small_graphs} we can write:
\begin{align}\label{eq:Uupsample}
U = \left[ \begin{array}{c}
 -\Larg^{-1}_r (\bar{\Omega}_r,\bar{\Omega}_r)\Larg_r (\bar{\Omega}_r, \Omega_r)\tilde{U}  \\
 \tilde{U} \end{array} \right] \nonumber \\
V = \left[ \begin{array}{c}
 -\Larg^{-1}_c (\bar{\Omega}_c,\bar{\Omega}_c)\Larg_c (\bar{\Omega}_c, \Omega_c)\tilde{V}  \\
 \tilde{V} \end{array} \right].
\end{align}
 Eqs. \eqref{eq:Uupsample}  involves solving  a sparse linear system. If each connected component of the graph has at least one labeled element, $\Larg_{r}(\bar{\Omega}_r,\bar{\Omega}_r)$ is full rank and invertible. If the linear system above is not large then one can directly use eq. \eqref{eq:Uupsample}. However, to avoid inverting the big matrix we can use the standard Preconditioned Conjugate Gradient (PCG) method to solve it.  Note that the eqs. (\ref{eq:Uupsample}) and even PCG can be implemented in parallel for every column of $U$ and $V$. This gives a significant advantage over the alternate decoder in terms of computation time. The cost of this decoder is $\mathcal{O}(O_l \mathcal{K}kn)$ where $O_l$ is the number of iterations for the PCG method. The columns of $U$ and $V$ are not normalized with the above solution, therefore, a unit norm normalization step is needed at the end. Once $U,V$ are determined, one can use $X = U\tilde{\Sigma}V^{\top}\sqrt{np / \rho_r \rho_c (1-\delta)}$  to determine the required low-rank matrix $X$. The decoder for approximate recovery is presented in Algorithm 2. 
 
 %is symmetric, we can use a Lancoz based approximation \cite{susnjara2015accelerated}. The  idea is based on the fact that 
% \begin{align*}
% \Larg_a^\dagger R = g(\Larg_a) R,
% \quad \text{where} \quad
% g(x)=\begin{cases}
% \frac{1}{x} & x>10^{-8}\\
% 0 & \mbox{otherwise.}
% \end{cases}
% \end{align*}

\begin{algorithm}
\caption{Subspace Upsampling based Approximate Decoder for low-rank recovery}
\label{algorithm2}
\begin{algorithmic}
\State INPUT: $\tilde{X} \in \Re^{\rho_c \times \rho_r}$, $\Larg_r \in \Re^{p \times p}$, $\Larg_c \in \Re^{n \times n}$
\State 1. do $SVD(\tilde{X}) =\tilde{U}\tilde{\Sigma}\tilde{V}^\top$
\State 2. find $k$ such that $\tilde{\Sigma}_{k,k}/\tilde{\Sigma}_{1,1} < 0.1$
\State 3. Solve eqs. \eqref{eq:approx_uv} for every column of $U,V$ as following:
\For{$i = 1, \dots k$} 
\State  solve $\min_{u_i} u_i^\top \Larg_r u_i \quad \text{s.t} ~ M_r u_i = \tilde{u}_i$ using PCG
\State  solve $\min_{v_i} v_i^\top \Larg_c v_i \quad \text{s.t} ~ M^\top_c v= \tilde{v}_i$ using PCG
\EndFor
\State 4. Set $u_i = u_i / \|u_i\|_F, v_i = v_i / \|v_i\|_F, \forall i = 1, \cdots, k$ 
\State 5. Set $\Sigma = \sqrt{\frac{np}{\rho_r \rho_c (1-\delta)}}\tilde{\Sigma}$
\State 6. Set $X = U \Sigma V^\top$ 
\State OUTPUT: The full low-rank $X \in \Re^{p \times n}$
\end{algorithmic}
\end{algorithm}

 Two other approximate decoders for low-rank recovery are presented in Appendix \ref{sec:approx_decoders}.

\section{\textbf{Decoder for clustering}}\label{sec:clustering}
As already mentioned earlier, PCA has been widely used for two types of applications: 1) low-rank recovery and 2) clustering.  Therefore, in this section, we present a method to perform clustering using our framework. 
 
For the clustering application we do not need the full low-rank matrix $X$. Thus, we propose to do k-means on the low-rank representation of the sampled data $\tilde{X}$ obtained using \eqref{eq:frpcag_small}, extract the cluster labels $\tilde{C}$ and then decode the cluster labels $C$ for $X$ on the graphs $\Larg_r$ and $\Larg_c$. 

Let $\tilde{C} \in \{0,1\}^{\rho_c \times k}$ be the cluster labels of $\tilde{X}$ (for $k$ clusters) which are obtained by performing k-means. 
Then,
\begin{align*}
\tilde{C}_{ij} = \left\{
\begin{array}{cc}
      1 & \text{if} ~ \tilde{x}_i \in ~ j^{th}~ \text{cluster} \\
      0 & \text{otherwise.}
\end{array} \right.
\end{align*}

Note that each of the columns $\tilde{c}_{i}$ of $\tilde{C}$ is the cluster indicator for one of the $k$ clusters. The goal now is to decode the cluster indicator matrix $C \in \{0,1\}^{n \times k}$. We refer to the Compressive Spectral Clustering (CSC) framework \cite{tremblay2016compressive}, where the authors solve a similar problem by arguing that each of the columns of $C$ can be obtained by assuming that it lies close to the $span(Q_{k_c})$, where $Q_{k_c}$ are the first $k_c$ Laplacian eigenvectors of the graph $G_c$. This requires solving the following convex minimization problem:

\begin{equation}
\min_C \|M^\top_c C - \tilde{C}\|^2_F + \gamma \tr(C^\top \Larg_c C)
\end{equation}
The above problem can be solved independently for each of the columns of $C$, thus,
\begin{equation}\label{eq:ci}
\min_{c_i} \|M^\top_c c_i - \tilde{c_i}\|^2_2 + \gamma c_i^\top \Larg_c c_i
\end{equation}
Furthermore, note that the graph $G_r$ is not required for this process. Eq.\eqref{eq:ci} gives a faithful solution for $c_i$ if the sampling operator $M_c$ satisfies the restricted isometry property RIP.  Thus, for any $\delta_c, \epsilon_c \in (0, 1)$, with probability at least $1 - \epsilon_c$, 
\begin{align}\label{eq:ripc2}
(1-\delta_c) \norm{w}_2^2 \leq \frac{n}{\rho_c} \norm{w^\top M_c}_2^2 \leq (1+\delta_c) \norm{w}_2^2
\end{align}
for all $w \in {span}(Q_{k_c})$ provided that
\begin{align}
\label{eq:cond21}
\rho_c \geq \frac{3}{\delta_c^2} \nu_{{k_c}}^2 \log\left(\frac{2k_c}{\epsilon_c}\right). 
\end{align}

This holds true as a consequence of Theorem \ref{thm:embedding} (eq.\eqref{eq:ripc} in the proof of Theorem \ref{thm:embedding} and Theorem 5 in \cite{puy2015random}).

Eq.\eqref{eq:ci} requires the tuning of a model parameter $\gamma$ which we want to avoid. Therefore, we use the same strategy as for the approximate low-rank decoder in Section \ref{sec:approx_decoder}. The cluster labels $\tilde{c}_i$ are not noisy because they are obtained by running $k$-means on the result of FRPCAG eq.\eqref{eq:frpcag_small}, which is robust to outliers. Thus, we set $\gamma = 0$ in eq.\eqref{eq:ci} and propose to solve the following problem: 
\begin{align}\label{eq:approx_c}
& \min_{c_i} c_i^{\top}\Larg_c c_i \quad \text{s.t:} \quad  M^{\top}_c c_i = \tilde{c_i}.
\end{align}
According to Lemma~\ref{lemma:graphupsampling}, the solution is given by:
\begin{equation}\label{eq:Cupsample}
c_i = \left[ \begin{array}{c}
 -\Larg^{-1}_c (\bar{\Omega}_c,\bar{\Omega}_c)\Larg_c (\bar{\Omega}_c, \Omega_c)\tilde{c_i}  \\
 \tilde{c_i} \end{array} \right].
\end{equation}
Ideally, every row of the matrix $C$ should have 1 in exactly one of the $k$ columns, indicating the cluster membership of that data sample. However, the solution $C \in \Rbb^{n \times k}$ obtained by solving the above problem is not binary. Thus, to finalize the cluster membership (one of the $k$ columns), we perform a maximum pooling for each of the rows of $C$, i.e,
\begin{align*}
{C}_{ij} \leftarrow \left\{
\begin{array}{cc}
      1 & \text{if} ~ C_{ij} = \max\{C_{ij} ~ \forall~ j = 1\cdots k\} \\
      0 & \text{otherwise.}
\end{array} \right.
\end{align*}

Algorithm 3 summarizes this procedure. 
\begin{algorithm}
\caption{Approximate Decoder for clustering}
\label{algorithm3}
\begin{algorithmic}
\State INPUT: $\tilde{X} \in \Re^{\rho_c \times \rho_r}$, $\Larg_c \in \Re^{n \times n}$
\State 1. do k-means on $\tilde{X}$ to get the labels $\tilde{C} \in \{0,1\}^{\rho_c \times k}$
\State 2. Solve eqs. \eqref{eq:approx_c} for every column of $C$ as following:
\For{$i = 1, \dots k$} 
\State  solve $\min_{c_i} c_i^\top \Larg_c c_i \quad \text{s.t} ~ M^\top_c c_i= \tilde{c}_i$ using PCG
\EndFor
\State 3. Set $C_{ij} = 1$ if $\max\{C_{ij} ~ \forall~ j = 1\cdots k\}$ and $0$ otherwise.
\State OUTPUT: cluster indicators for $X$: $C \in \{0,1\}^{n \times k}$
\end{algorithmic}
\end{algorithm}

\begin{thm}\label{thm:approxdecoderclustering}
Let $M_c$ be such that \eqref{eq:ripc2} holds. Let also $c_i^*$ be the solution of \eqref{eq:approx_c}  with $\tilde{c_i} = M^\top_c {\bar{c}_i}$, where $\bar{c}_i \in span(Q_{k_c})$, $i = 1, \cdots, n$. We have:
\begin{align*}
\|{c_i}^* - \bar{c}_i\|_2 &
\leq 
\sqrt{\frac{n}{\rho_c(1-\delta_c)}} \sqrt{\frac{\lambda_{k_c}}{\lambda_{k_c+1}}} \norm{\bar{c}_i}_2 
\end{align*}
where ${c}_i^* = Q_{k_c} Q_{k_c}^\top c_i$. 
\end{thm} 
\begin{proof}
The proof directly follows from the proof of Theorem 3.2 in  \cite{puy2015random}. These steps have been repeated in the proof of Theorem \ref{thm:approxdecoder} in Appendix \ref{sec:approxdecoder} as well. Using $c^*_i = \bar{u}^*_i$, $\bar{c}_i = \bar{u}_i$, $n = p, \rho_c = \rho_r$ in eq.\eqref{eq:u_1} one can get theoretical guarantees for eq. \eqref{eq:ci}. Then, by using $\tilde{e}^u_i = 0$ and $\gamma = 0$ we get the result of above theorem.
\end{proof}%\vspace{-0.4cm}

\begin{table*}[htbp]
\caption{Summary of CPCA and its computational complexity for a dataset $Y \in \Re^{p \times n}$. Throughout we assume that $\K, k, \rho_r, \rho_c, p \ll n$. }
\centering
\resizebox{1.0\textwidth}{!}{\begin{tabular}[t]{| c | c | c |}
\hline
 \textbf{Steps}  & \textbf{\centering The Complete CPCA Algorithm}  & \textbf{Complexity}  \\\hline
 1   & Construct graph Laplacians between the rows $\Larg_r$ and columns $\Larg_c$ of $Y$ using Section \ref{sec:graphs}.    & $\mathcal{O}(np \log(n))$ \\\hline
2   & Construct row and column sampling matrices $M_r \in \Re^{\rho_r \times p}$ and $M_c \in \Re^{n \times \rho_c}$ satisfying \eqref{eq:Mc} and theorem \ref{thm:embedding}  & -- \\\hline 
3  & Sample the data matrix $Y$ as $\tilde{Y} = M_r Y M_c$  & -- \\\hline
4  & Construct the new graph Laplacians between the rows $\tilde{\Larg}_r$ and columns $\tilde{\Larg}_c$ of $\tilde{Y}$ using Section \ref{sec:small_graphs}. & $\mathcal{O}(O_l \mathcal{K}n)$\\\hline
5  & Solve FRPCAG \eqref{eq:frpcag_small} using Algorithm 1 to get the low-rank $\tilde{X}$  & $\mathcal{O}(I \rho_r \rho_c \mathcal{K})$ \\\hline
 6 & \textbf{For low-rank recovery}: Decode $X$ from $\tilde{X}$ using the approximate decoder Algorithm 2 & $\mathcal{O}(O_l n k \mathcal{K} )$ \\\hline
 7  & \textbf{For clustering}:  Decode the cluster labels $C$ for $X$  using the semi-supervised label propagation (Algorithm 3) & $\mathcal{O}(O_l n k )$ \\\hline
\end{tabular}}
\label{tab:cpca}
\end{table*}

%\section{Complete CPCA Algorithm \& Complexity}
\textbf{Computational Complexity:} A summary of all the decoders and their computational complexities is presented in Table \ref{tab:decoders} of Appendix \ref{sec:summary_decoders}. The complete CPCA algorithm and the computational complexities of different steps are presented in Table \ref{tab:cpca}. For $\K, k, \rho_r, \rho_c, p \ll n$ CPCA algorithm scales as $O(n k \mathcal{K})$ per iteration. Thus, assuming that the row and column graphs are available from external source, a speed-up of $p/k$ per iteration is obtained over FRPCAG and $p^2/k$ over RPCA. A detailed explanation regarding the calculation of complexities of CPCA and other models is presented in Table \ref{tab:complexity} and Appendix \ref{sec:summary_decoders}. 

\textbf{Memory Requirements:} We compare the memory requirements of CPCA with FRPCAG. For a matrix $Y \in \Rbb^{p \times n}$, FRPCAG and CPCA require the construction of two graphs $G_r, G_c$ whose Laplacians $L_r \in \Rbb^{p \times p}, L_c \in \Rbb^{n \times n}$ are used in the core algorithm. However, these Laplacians are sparse, therefore the memory requirement for $L_r, L_c$ is $\mathcal{O}(\K (|\E_r| + |\E_c|))$ respectively.  The core algorithm of FRPCAG requires operation on the full matrix $Y$  and the graph Laplacians $L_r, L_c$. As $|\E_r| \approx \K p$ and $|\E_c| \approx \K n$ therefore, the memory requirement for the regularization terms $\tr(X L_c X^{\top})$ and $\tr(X^{\top} L_r X)$ is  $\mathcal{O}(\K np )$. For the CPCA algorithm, assuming $n > p$ and  letting $\rho_r = p/b$ and $\rho_c = n/a$, the complexity of FRPCAG on the sampled data is $\mathcal{O}(\K np/(ab) )$ and the approximate decode stage for subspaces of dimension $k$ is $\mathcal{O}(\K n k)$. Thus the overall memory requirement of CPCA is $\mathcal{O}(\K n(p/(ab)+k))$. As compared to FRPCAG, an improvement of $pab/(p + k ab)$ is obtained. For example for $n = 1000, p = 200, a = 10, b = 1, k = 10$, a reduction of approximately $6.6$ times is obtained.

{\textbf{Convergence of CPCA:} The CPCA based algorithm (Table \ref{tab:cpca}) has two main steps: 1) FRPCAG on the compressed data matrix and 2) low-rank matrix or cluster label decoding. FRPCAG is solved by the FISTA (Algorithm 1) and the decode step is solved using the PCG method. Both of these methods have been well studied in terms of their convergence guarantees. More specifically, one can refer to \cite{beck2009fast} for a detailed study on FISTA and \cite{axelsson1986rate} for PCG. Therefore, we do not include the convergence analysis here for brevity.}

\begin{table*}[htbp]
\footnotesize
\caption{Clustering error of USPS datasets for different PCA based models. The best results per column are highlighted in bold and the 2nd best in blue. NMF and GNMF require non-negative data so they are not evaluated for USPS because USPS is also negative.}
\centering
\resizebox{1.0\textwidth}{!}{\begin{tabular}[t]{| c || c || c || c | c | c | c || c | c | c | c || c | c | c | c | }\hline
\textbf{Dataset} & \textbf{Model}  &  \textbf{no noise} & \multicolumn{4}{c ||}{\textbf{Gaussian noise}} & \multicolumn{4}{c ||}{\textbf{Laplacian noise}}  & \multicolumn{4}{c ||}{\textbf{Sparse noise}} \\\cline{4-15}
 &  &   &  \textbf{5\%} & \textbf{10\%}  & \textbf{15\%}  & \textbf{20\%}  &  \textbf{5\%} & \textbf{10\%}  & \textbf{15\%}  & \textbf{20\%}  &  \textbf{5\%} & \textbf{10\%}  & \textbf{15\%}  & \textbf{20\%}  \\\hline
& k-means  &   0.31  & 0.31 & 0.31 & 0.33 & 0.32 & 0.32 & 0.30 & 0.36 & 0.37 & 0.40 & 0.45 & 0.53 & 0.73 \\\cline{2-15}
& LLE    &  0.40  & 0.34  & 0.32 & 0.35 &  0.24 & 0.40 & 0.40 &  0.33 & 0.36 & \n{0.23} & 0.30 & 0.33 & 0.37 \\\cline{2-15}
& LE & 0.38  & 0.38  &  0.38  & 0.36 & 0.35 & 0.38 & 0.38 &  0.38 & 0.38 & 0.32 & 0.33 & 0.36 & 0.48 \\\cline{2-15}
USPS  & PCA  & 0.27  & 0.29 & 0.25 & 0.28 & 0.26 & 0.29 & 0.29 &  0.28 & 0.24 & 0.29 & 0.26 & 0.26 & 0.28 \\\cline{2-15}
small & MMF & \n{0.21}  & \textbf{0.20}  & \n{0.21} & 0.22 & \n{0.21} & \n{0.21} & \n{0.21} &  0.22 & 0.21 & 0.27 & \n{0.23} & 0.25 & {0.27} \\\cline{2-15}
($n = 3500$& GLPCA & \textbf{0.20} & \textbf{0.20} & \n{0.21} & 0.23 & 0.23 & \n{0.21} & \n{0.21} &  0.22 & 0.21 & 0.26 & 0.24 & 0.24 & 0.28 \\\cline{2-15}
$p = 256$)& RPCA &  0.26 & \n{0.25} & 0.23 & 0.24 & 0.22 & 0.26 & 0.26 &  0.25 & 0.24 & 0.26 & 0.24 & \n{0.23} & 0.30 \\\cline{2-15}
& RPCAG & \textbf{0.20} & \textbf{0.20} & \n{0.21} & \n{0.20} & \n{0.21} & \textbf{0.20} & \n{0.21} &  \n{0.21} & \n{0.21} & \textbf{0.21} & \textbf{0.22} & \n{0.23} & \n{0.25}  \\\cline{2-15}
 & FRPCAG &  \textbf{0.20} & \textbf{0.20} & \textbf{0.20} & \textbf{0.19} & \textbf{0.20} & \textbf{0.20} & \textbf{0.19} & \textbf{0.17} & \textbf{0.17} & \textbf{0.21} & \textbf{0.22} & \textbf{0.22} & \textbf{0.23} \\\cline{2-15}
 & CPCA (2,1)  & \textbf{0.20}  & \textbf{0.20}  & \n{0.21} & \n{0.20} & 0.22 & \textbf{0.20} & 0.22 & 0.22 & \n{0.21} & \n{0.23} & \n{0.23} & 0.25 & 0.28 \\\hline \hline

& k-means  & 0.26 & 0.26 & 0.26 & 0.26 & 0.28 & 0.27 & 0.26 & 0.26 & 0.26 & 0.26 & 0.25 & 0.34 & 0.30 \\\cline{2-15}
& LLE    & 0.51 & 0.29 & 0.22 &  0.21 &  0.22 & 0.22 & 0.22 & 0.22 & 0.21 & 0.22 & \textbf{0.19} & 0.26 & 0.31 \\\cline{2-15}
& LE & 0.33 & 0.32 & 0.32 &  0.27 & 0.27 &  0.32 & 0.34 & 0.31 & 0.34 & 0.35 & 0.44 & 0.49 & 0.53 \\\cline{2-15}
USPS  & PCA  & 0.21 & 0.21 &  0.21 & 0.22 & 0.21 & 0.21 & 0.21 & 0.22 & 0.21 & 0.22 & 0.23 & 0.23 & \n{0.23} \\\cline{2-15}
large & MMF & 0.24 & 0.23 & 0.23 & 0.24 & 0.24 & 0.19 & 0.23 & 0.22 & 0.23 & 0.24 & 0.25 & 0.26 & 0.26 \\\cline{2-15}
($n = 10000$ & GLPCA & \n{0.16} & \n{0.16} & \n{0.17} & \n{0.17} & \n{0.16} & {0.17} & \n{0.15} & \n{0.15} & \n{0.17} & \textbf{0.18} & \textbf{0.19} & \textbf{0.21} & \n{0.23} \\\cline{2-15}
 $p = 256$)& FRPCAG & \textbf{0.15} & \n{0.16} & \n{0.17} & \textbf{0.15} & \textbf{0.14} & \n{0.16} & 0.16 & 0.16 & \n{0.17} & \textbf{0.18} & \n{0.21} & \textbf{0.21} & \textbf{0.21} \\\cline{2-15}
& CPCA (10,1)  & \textbf{0.15} & \textbf{0.14} &  \textbf{0.14} &  \textbf{0.15} & \textbf{0.14} & \textbf{0.14} & \textbf{0.14} & \textbf{0.13} & \textbf{0.14} & \textbf{0.18} & \textbf{0.19} & \n{0.22} & 0.24 \\\hline\hline

& K-means  & 0.51 &  0.51 & 0.52 & 0.51 & 0.58 & 0.52 & 0.51 & 0.52 & 0.52 & 0.80 & 0.88 & 0.88 & 0.88\\\cline{2-15}
& PCA  & 0.43 & 0.43 & 0.41 & 0.43 & 0.42 & 0.43 & 0.42 & 0.38 & 0.43 & 0.42 & 0.42 & 0.42 & 0.44 \\\cline{2-15}
& MMF & \n{0.33} & \n{0.34} & \n{0.34} & \n{0.34} & \textbf{0.33} & \n{0.34} & \textbf{0.33} & 0.37 & \n{0.34} & 0.40 & {0.38} & \n{0.38} & \n{0.42} \\\cline{2-15}
MNIST & GLPCA & 0.38 & 0.36 & 0.37 & 0.35 & 0.38 & 0.39 & \n{0.36} & \n{0.36} & 0.36 & \n{0.37} & 0.39 & \n{0.38} & \textbf{0.39} \\\cline{2-15}
small & PCAG-(1,0)  &  0.40 & 0.40 & 0.41 &  0.40 & 0.40 & 0.40 & 0.40 & 0.41 & 0.40 & \n{0.37} & \n{0.37} & 0.39 & 0.44 \\\cline{2-15}
( $n$ = 1000 & FRPCAG &  \textbf{0.32} & \textbf{0.33} & \textbf{0.32} & \textbf{0.33} & \textbf{0.33} & \textbf{0.32} & \textbf{0.33} &  \textbf{0.32} & \textbf{0.32} & \textbf{0.33} & \textbf{0.36} & \textbf{0.35} & \textbf{0.39} \\\cline{2-15}
 $p$ = 784) & CPCA (5,1)  &  0.39 & 0.38 & 0.37 & 0.38 & 0.38 & 0.39 & 0.39 & \n{0.36} & 0.37 & 0.39 & 0.40 & 0.41 & 0.50  \\\hline\hline
\end{tabular}}
\label{tab:resultsUSPS}
\end{table*}

\section{\textbf{Experimental Results}}\label{sec:results}
 We perform two types of experiments corresponding to two applications of PCA 1) Data clustering and 2) Low-rank recovery using two open-source toolboxes: the UNLocBoX \cite{perraudin2014unlocbox} and the GSPBox \cite{perraudin2014gspbox}. 

\subsection{\textbf{Clustering}}
\subsubsection{\textbf{Experimental Setup}}
\textbf{Datasets}: We perform our clustering experiments on 5 benchmark databases (as in \cite{shahid2015robust, shahid2015fast}): 
CMU PIE, ORL, YALE, MNIST and USPS. For the USPS and ORL datasets, we further run two types of experiments 1) on subset of datasets and 2) on full datasets. The experiments on the subsets of the datasets take less time so they are used to show the efficiency of our model for a wide variety of noise types. The details of all datasets used are provided in Table~\ref{tab:datasets} of Appendix \ref{sec:summary_decoders}. 

\textbf{Noise \& Errors}: {CPCA is a memory and computationally efficient alternative for FRPCAG. An important property of FRPCAG is its robustness to noise and outliers, just like RPCA. Therefore, it is important to study the performance of CPCA under noise and corruptions similar to those for FRPCAG and RPCA}. To do so we add  3 different types of noise in all the samples of datasets in different experiments: 1) Gaussian noise and 2) Laplacian noise with standard deviation ranging from 5\% to 20\% of the original data 3) Sparse noise (randomly corrupted pixels) occupying 5\% to 20\% of each data sample.

\textbf{Comparison with other methods}: We compare the clustering performance of CPCA with 11 other models including: 1) k-means on original data  2) Laplacian Eigenmaps (LE) \cite{belkin2003laplacian}  3) Locally Linear Embedding (LLE) \cite{roweis2000nonlinear} 4) Standard PCA 5) Graph Laplacian PCA (GLPCA) \cite{jiang2013graph} 6) Manifold Regularized Matrix Factorization (MMF) \cite{zhang2013low}  7) Non-negative Matrix Factorization (NMF) \cite{lee1999learning} 8) Graph Regularized Non-negative Matrix Factorization (GNMF) \cite{cai2011graph} 9) Robust PCA (RPCA) \cite{candes2011robust} 10) Robust PCA on Graphs (RPCAG) \cite{shahid2015robust} and 11) Fast Robust PCA on Graphs (FRPCAG) \cite{shahid2015fast}.  RPCA and RPCAG are not used for the evaluation of MNIST, USPS large and ORL large datasets due to computational complexity of these models.

\textbf{Pre-processing}: All datasets are transformed to zero-mean and unit standard deviation along the features / rows. For MMF the samples are additionally normalized to unit-norm. For NMF and GNMF only the unit-norm normalization is applied to all the samples of the dataset as NMF based models can only work with non-negative data.

\textbf{Evaluation Metric}: We use \textit{clustering error} as a metric to compare the clustering performance of various models.  The clustering error for LE, PCA, GLPCA, MMF, NMF and GNMF is evaluated by performing k-means on the principal components ${V}$ (note that these models explicitly learn $V$, where $X = U\Sigma V^\top$).   The clustering error for RPCA, RPCAG and FRPCAG is determined by performing k-means directly on the low-rank $X$. For our CPCA method, k-means is performed  on the small low-rank $\tilde{X}$ and then the labels for full $X$ are decoded using Algorithm 3.

\begin{table}[h!]
\footnotesize
\caption{Clustering error of ORL datasets for different PCA based models. The best results per column are highlighted in bold and the 2nd best in blue.}
\centering
\resizebox{0.45\textwidth}{!}{\begin{tabular}[t]{| c || c || c || c | c | c | c | }\hline
\textbf{Data} & \textbf{Model}  &  \textbf{no noise} & \multicolumn{4}{c ||}{\textbf{Gaussian noise}} \\\cline{4-7}
 \textbf{set}&  &   &  \textbf{5\%} & \textbf{10\%}  & \textbf{15\%}  & \textbf{20\%}    \\\hline
& k-means  &  0.40 & 0.43 & 0.43 & 0.45 & 0.44  \\\cline{2-7}
& LLE    & 0.26 &  0.26 & 0.26 &  0.26 & 0.19  \\\cline{2-7}
& LE &  0.21 &  0.18 & 0.19 & 0.19 & 0.19  \\\cline{2-7}
& PCA  & 0.30 & 0.30 & 0.32 & 0.34 & 0.32   \\\cline{2-7}
O& MMF & 0.21 &  0.20 & 0.18 & 0.18 & 0.17  \\\cline{2-7}
R& GLPCA & \textbf{0.14} & \textbf{0.13} & \textbf{0.13} & \textbf{0.13} & \textbf{0.14}  \\\cline{2-7}
L& NMF & 0.31 & 0.34 & 0.29 & 0.31 & 0.34   \\\cline{2-7}
& GNMF & 0.29 & 0.29 & 0.29 & 0.31 & 0.29   \\\cline{2-7}
s& RPCA & 0.36 & 0.34 & 0.33 & 0.35 & 0.36  \\\cline{2-7}
m & RPCAG & 0.17 & 0.17 & 0.17 & 0.16 & \n{0.16}  \\\cline{2-7}
a& FRPCAG & \n{0.15} &  \n{0.15} & \n{0.15} & \n{0.14} & \n{0.16}   \\\cline{2-7}
l& CPCA (2,2)  & 0.23 &   0.23 & 0.23 & 0.24 & 0.25   \\\cline{2-7}
l& CPCA (1,2) & 0.17 & 0.17 & 0.17 &  \n{0.14} & 0.17  \\\hline\hline

O& k-means  & 0.49 & 0.50 &  0.51 & 0.51 & 0.51 \\\cline{2-7}
R& LLE    &  0.28 & 0.27 &  0.27 & 0.24 & 0.25   \\\cline{2-7}
L& LE &  0.24 & 0.25 & 0.25 & 0.24 & 0.25 \\\cline{2-7}
& PCA  &  0.35 &  0.34 &  0.35 & 0.36 & 0.36  \\\cline{2-7}
l& MMF &  0.23 &  0.23 & 0.23 & 0.24 & 0.24  \\\cline{2-7}
a & GLPCA & \n{0.18} &  \n{0.18} & \n{0.18} &  \n{0.19} & \n{0.19} \\\cline{2-7}
 r& NMF &  0.36 & 0.33 &  0.36 & 0.32 & 0.36 \\\cline{2-7}
g& GNMF &  0.34 & 0.37 &  0.36 & 0.39 & 0.39 \\\cline{2-7}
e& FRPCAG &  \textbf{0.17} & \textbf{0.17} & \textbf{0.17} & \textbf{0.17} & \textbf{0.17} \\\cline{2-7}
& CPCA (2,2)  & 0.21 &  0.21 & 0.21 & 0.23 & 0.22 \\\hline\hline
\end{tabular}}
\label{tab:resultsORL}
\end{table}

\begin{table}[htbp]
\footnotesize
\caption{Clustering error of CMU PIE and YALE datasets for different PCA based models. The best results per column are highlighted in bold and the 2nd best in blue. }
\centering
\resizebox{0.45\textwidth}{!}{\begin{tabular}[t]{| c || c || c || c | c | c | c |}\hline
\textbf{Data} & \textbf{Model}  &  \textbf{no noise} & \multicolumn{4}{c ||}{\textbf{Gaussian noise}}  \\\cline{4-7}
\textbf{set} &  &   &  \textbf{5\%} & \textbf{10\%}  & \textbf{15\%}  & \textbf{20\%}  \\\hline
& k-means  & 0.76 & 0.76 & 0.76 & 0.75 & 0.76  \\\cline{2-7}
& LLE    &  0.47 &  0.50 & 0.52 & 0.50 & 0.55  \\\cline{2-7}
& LE & 0.60 &  0.58 & 0.59 & 0.58 & 0.60  \\\cline{2-7}
& PCA  &  0.27 & 0.27 & 0.27 & 0.27 & 0.29  \\\cline{2-7}
C& MMF &  0.67 & 0.67 & 0.67 & 0.66 &  0.67 \\\cline{2-7}
M& GLPCA & 0.37 & 0.39 & 0.37 & 0.38 & 0.38  \\\cline{2-7}
U & NMF & \n{0.24} & {0.27} & \textbf{0.24} &  \n{0.25} & 0.27  \\\cline{2-7}
P& GNMF &  0.58 & 0.59 & 0.56 &  0.58 & 0.59  \\\cline{2-7}
I& RPCA &  0.39 & 0.38 & 0.41 & 0.41 & 0.38  \\\cline{2-7}
E& RPCAG & \n{0.24} & \textbf{0.24} & \textbf{0.24} & \textbf{0.24} & \n{0.25}  \\\cline{2-7}
& FRPCAG &  \textbf{0.23} & \textbf{0.24} & \textbf{0.24} & \textbf{0.24} &  \textbf{0.24}   \\\cline{2-7}
& CPCA (2,1) & 0.26 &  \n{0.26} & \n{0.26} & 0.28 & 0.27  \\\cline{2-7}
& CPCA (2,2) & 0.28 & 0.29 & 0.29 & 0.30 & 0.30  \\\hline\hline

& k-means & 0.76 & 0.76 & 0.76 & 0.76 & 0.77 \\\cline{2-7}
& LLE   &  0.51 & 0.47 &  0.46 & 0.49 & 0.50 \\\cline{2-7}
& LE & 0.52 & 0.56 & 0.55 & 0.54 & 0.54 \\\cline{2-7}
& PCA  &  0.53 & 0.52 & 0.53 & 0.56 & 0.55  \\\cline{2-7}
& MMF &  0.58 & 0.59 & 0.56 & 0.57 & 0.58  \\\cline{2-7}
Y& GLPCA & 0.47 & 0.46 & 0.47 & 0.45 & 0.48  \\\cline{2-7}
A & NMF &  0.57 & 0.58 &  0.59 & 0.57 & 0.59  \\\cline{2-7}
L & GNMF &  0.59 &  0.59 & 0.61 & 0.59 & 0.60  \\\cline{2-7}
E & RPCA &  0.45 & 0.48 & 0.46 & 0.46 & 0.48  \\\cline{2-7}
& RPCAG &  \textbf{0.40} & \n{0.40} & \n{0.41} & \n{0.41} & \n{0.41}  \\\cline{2-7}
& FRPCAG &  \textbf{0.40} & \textbf{0.37} & \textbf{0.40} & \textbf{0.40} & \textbf{0.40} \\\cline{2-7}
& CPCA (2,2)  & \n{0.43} & 0.45 & 0.42 & 0.46 & 0.43  \\\hline \hline
\end{tabular}}
\label{tab:resultsCMUPIE_YALE}
\end{table}

\textbf{Parameter Selection}: To perform a fair validation for each of the models we use a range of values for the model parameters as presented in Table~\ref{tab:models_param} of Appendix \ref{sec:summary_decoders}. For a given dataset, each of the models is run for each of the parameter tuples in this table {and the best clustering error is reported}. Furthermore, PCA, GLPCA, MMF, NMF and GNMF are non-convex models so they are run $10$ times for each of the parameter tuple. RPCA, RPCAG, FRPCAG and CPCA based models  are convex so they are run only once. {Although our CPCA approach is convex, it involves a bit of randomness due to the sampling step. Due to the extensive nature of the experimental setup, most of the clustering experiments are performed under one sampling condition. However, it is interesting to study the variation of error under different sampling scenarios. For this purpose we perform some additional experiments on the USPS dataset}. For our proposed CPCA, we use a convention $CPCA (a,b)$, where $a$ and $b$ denote the downsampling factors on the columns and  rows respectively. A uniform sampling strategy is always used for CPCA.

\textbf{Graph Construction:} The $\K$-nearest neighbor graphs $G_r, G_c$ are constructed using FLANN \cite{muja2009fast} as discussed in Section~\ref{sec:graphs}. The small graphs $\tilde{G}_r, \tilde{G}_c$ can also be constructed using FLANN or the Kron reduction strategy of Section \ref{sec:small_graphs}. For all the experiments reported in this paper we use   $\mathcal{K}$-nearest neighbors = 10 and Gaussian kernel for the adjacency matrices $W$. The smoothing parameters $\sigma^2$ for the Gaussian kernels are automatically set to the average distance of the $\mathcal{K}$-nearest neighbors.

\subsubsection{\textbf{Discussion on clustering performance}} We point out here that the purpose of our clustering experiments is three-fold:
\begin{itemize}
\item To show the efficiency of CPCA for a wide variety of noise and errors and downsampling. 
\item To study the conditions under which CPCA performs worse than the other models. 
\item To study the variation of performance under different sampling scenarios.  
\end{itemize}
For this purpose, we test CPCA under a variety of  downsampling for different datasets. Cases with $p \ll n$ and $n \ll p$ carry special interest. Therefore, we present our discussion below in the light of the above goals.

Tables~\ref{tab:resultsUSPS}, \ref{tab:resultsORL}, \ref{tab:resultsCMUPIE_YALE} \& \ref{tab:mnistbig}   present the clustering results for USPS small, USPS large, MNIST large, MNIST small, ORL small, ORL large, CMU PIE and YALE datasets. Note that not all the models are run for all the datasets due to computational constraints. The best results are highlighted in bold and the second best in blue. From Table \ref{tab:resultsUSPS} for the USPS dataset, it is clear that our proposed CPCA model attains comparable clustering results to the state-of-the-art RPCAG and FRPCAG models and better than the others in most of the cases. Similar observation can be made about the MNIST large dataset from Table \ref{tab:mnistbig} in comparison to FRPCAG. 

 It is important to note that for the USPS and MNIST datasets $p \ll n$. Thus, for the USPS dataset, the compression is only applied along the columns ($n$) of the dataset. This compression results in clustering error which is comparable to the other state-of-the-art algorithms. As $p = 256$ for the USPS dataset, it was observed that even a 2 times downsampling on $p$ results in a loss of information and the clustering quality deteriorates. The same observation can be made about ORL small, ORL large and YALE datasets from Tables   \ref{tab:resultsORL}, \ref{tab:resultsCMUPIE_YALE} for CPCA (2,2). i.e, two times downsampling on both rows and columns. On the other hand the performance of CPCA (1,2) is reasonable for the ORL small dataset. Recall that CPCA (a,b) means a downsampling by $a$ and $b$ across columns and rows (samples and features). Also note that for ORL dataset $n \ll p$.
 
 Finally, we comment about the results on MNIST small dataset ($p = 784, n = 1000$) from Table \ref{tab:resultsUSPS}. It is clear that FRPCAG (no compression) results in the best performance.  CPCA (5,1) results in a highly undersampled dataset which does not capture enough variations in the MNIST small dataset to deliver a good clustering performance. {This particular case supports the fact that compression does not always yield a good performance at the advantage of reduced complexity. Therefore, we study this phenomena below.}
 
 The above findings for the MNIST dataset are intuitive as it only makes sense to compress both rows  and columns in our CPCA based framework if a reasonable speed-up can be obtained without compromising the performance, i.e, \textit{if both $n$ and $p$ are large}. If either $n$ or $p$ is small then one might only apply compression along the larger dimension, as the compression on the smaller dimension would not speed up the computations significantly. For example, for the USPS dataset, a speed-up of $p / k = 256/10 \approx 25$ times would be obtained over FRPCAG by compressing along the samples (columns $n$) only without  a loss of clustering quality. Our experiments showed that this speed up increased upto 30  by compressing along the features but with a loss of performance (The results are not presented for brevity).

Tables~\ref{tab:resultsUSPS}, \ref{tab:resultsORL}, \ref{tab:resultsCMUPIE_YALE} \& \ref{tab:mnistbig} also show that CPCA is quite robust to a variety of noise and errors in the dataset. Even in the presence of higher levels of Gaussian and Laplacian noise, CPCA performs  comparable to other methods for the USPS dataset. {Thus, CPCA tends to preserve the robustness property of FRPCAG. This will also be clear from the low-rank recovery experiments in the next section.}

\subsubsection{\textbf{Computation Time vs Performance \& Sampling}}
It is interesting to compare 1) the time needed for FRPCAG and CPCA to perform clustering  2) the corresponding clustering error and 3) the sub-sampling rates in CPCA. Table \ref{tab:mnistbig} shows such a  comparison for 70,000 digits of MNIST with (10, 2) times downsampling on the (columns, rows) respectively for CPCA. The time needed by CPCA is an order of magnitude lower than   FRPCAG.  Note that the time reported here does not include the  construction  of graphs $G_r, G_c$ as both methods use the same graphs. Furthermore, these graphs can be constructed in the order of a few seconds if parallel processing is used. The time for CPCA includes steps 2 to 5 and 7 of Table \ref{tab:cpca}.  {For the information about the graph construction time, please refer to Table \ref{tab:time} and the discussion thereof.}

Table~\ref{tab:time} presents the computational time and number of iterations for the convergence of  CPCA, FRPCAG, RPCAG \& RPCA on different sizes and dimensions of the datasets. We also present the time needed for the graph construction. The computation is done on a single core machine with a 3.3 GHz processor without using any distributed or parallel computing tricks. An $\infty$ in the table indicates that the algorithm did not converge in 4 hours. It is notable that our model requires a very small number of iterations to converge irrespective of the size of the dataset. Furthermore, the model is orders of magnitude faster than RPCA and RPCAG. This is clearly observed from the experiments on MNIST dataset  where our proposed model is 100 times faster than RPCAG. Specially for MNIST dataset with 25000 samples, RPCAG and RPCA did not converge even in 4 hours whereas CPCA converged in less than a minute. 

\begin{table}[htbp]
\caption{ Clustering error and computational times of FRPCAG and CPCA on MNIST  large dataset (784 $\times$ 70,000).}
\centering
\resizebox{0.35\textwidth}{!}{\begin{tabular}[t]{| c | c | c | }
\hline
  \textbf{Model}        &   \textbf{FRPCAG} &  \textbf{CPCA (10,2)}  \\\hline
\textbf{Error}  &  0.25     & 0.24    \\\hline
\textbf{time (secs)}  & 350  & 58  \\\hline
\end{tabular}}
\label{tab:mnistbig}
\end{table}

\begin{table*}[htbp]
\caption{Computation times (in seconds) for graphs $G_{r}$, $G_{c}$, FRPCAG, CPCA, RPCAG, RPCA and the number of iterations to converge for different datasets. The computation is done on a single core machine with a 3.3 GHz processor without using any distributed or parallel computing tricks. $\infty$ indicates that the algorithm did not converge in 4 hours. }
\centering
\resizebox{1.0\textwidth}{!}{\begin{tabular}[t]{| c | c | c | c | c | c | c | c | c | c | c | c | c | c | c |}
\hline
 \textbf{Dataset}  & \textbf{Samples}  & \textbf{Features}  & \textbf{Classes} & \multicolumn{2}{c |}{\textbf{Graphs}} & \multicolumn{2}{c | }{\textbf{FRPCAG}} & \multicolumn{3}{c | }{\textbf{CPCA}}  &  \multicolumn{2}{c | }{\textbf{RPCAG}} & \multicolumn{2}{c |}{\textbf{RPCA}} \\\cline{5-15}
 &     &      &      &  \textbf{$G_r$}  & \textbf{$G_c$}  & \textbf{time}  & \textbf{Iters}  & \textbf{(a,b)} & \textbf{time}  & \textbf{Iters}  & \textbf{time}  & \textbf{Iters} & \textbf{time}  & \textbf{Iters} \\\hline
 MNIST  &  5000    &   784    & 10 & 10.8   &  4.3  &   13.7   & 27  & (5,1)  & 5  & 30 & 1345  & 325 & 1090   & 378 \\\hline
 MNIST & 15000   & 784       & 10    &   32.5   &  13.3   & 35.4 & 23 &  (5,1)  & 13  & 25 & 3801  & 412   &  3400 & 323 \\\hline
 MNIST & 25000   & 784       & 10 & 40.7   & 22.2   & 58.6   & 24  & (10,1)   &  20  & 37  & $\infty$   &  $\infty$   &  $\infty$  & $\infty$ \\\hline
 ORL  & 300      & 10304   & 30 &  1.8   & 56.4    & 24.7 & 12 & (2,1)  & 14  &  15 & 360  & 301   & 240  &  320  \\\hline
 USPS & 3500     &  256   & 10  & 5.8   & 10.8   & 21.7  & 16  & (10,1)  & 12 & 31 & 900  &  410  &  790  & 350 \\\hline
% US census  & 2.5 million & 68   & -  & 540  & 42.3  & 3900 & 200  &  $\infty$ &  $\infty$ &  $\infty$ &  $\infty$ \\\hline
\end{tabular}}
\label{tab:time}
\end{table*}

\begin{figure}
\includegraphics[width=0.4\textwidth]{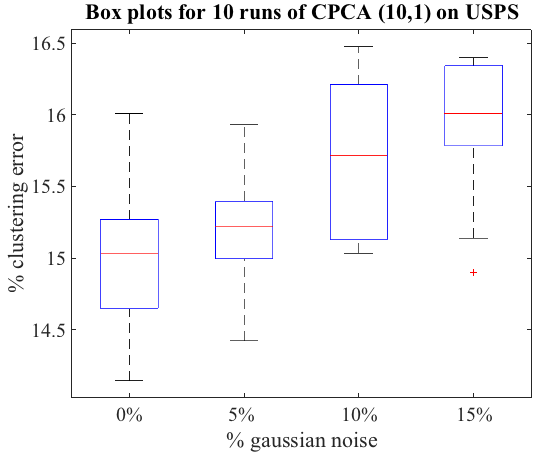}
\caption{Box plots for clustering error over 10 random sampling runs of CPCA (10,1) for the full USPS dataset ($256 \times 10000$) with increasing levels of Gaussian noise. For each run CPCA is evaluated for the full parameter grid $\gamma^{'}_r, \gamma^{'}_c \in (0,30)$ and the minimum clustering error is considered. A slight increase in the average clustering error with Gaussian noise shows that CPCA is quite robust to sampling and noise. }
\label{fig:sampling_effect} 
\end{figure}

\begin{table*}[htbp]
\captionof{table}{Preservation of the rank of the datasets in the compressed low-rank $\tilde{X}$ determined by solving FRPCAG \eqref{eq:frpcag_small}. }
\centering
\resizebox{0.75\textwidth}{!}{\begin{tabular}[t]{| c | c | c | c | c | c | c | c | c |}
\hline
\textbf{Dataset} &  \textbf{downsampling factor}  & \textbf{actual rank} & \multicolumn{6}{c ||}{\textbf{Rank after FRPCAG on sampled matrix}} \\\hline
& $(columns, rows)$ &  &  \multicolumn{3}{c ||}{\textbf{Gaussian noise}} & \multicolumn{3}{c ||}{\textbf{Laplacian noise}}   \\\hline
&  &     &             \textbf{5\%}  & \textbf{10\%}  & \textbf{15\%}  & \textbf{5\%}  & \textbf{10\%}  & \textbf{15\%} \\\hline
ORL large   & (2,1)  & 40   &   41      & 41     & 41       & 41       & 41    & 42   \\\hline
USPS large  & (10,2) & 10   &   10      &  11    & 11       &   11    &  11    &  11 \\\hline
MNIST     &  (10,2) & 10    &   11      & 11    & 11       &   11    &  11    &  11 \\\hline
CMU PIE  & (2,1)  &   30     &   31      & 31    & 31       &   31    &  31    &  32 \\\hline
YALE   &  (2,1) &     11     &   11      & 11    & 11       &   11    &  11    &  11 \\\hline
\end{tabular}}
\label{tab:rank}
\end{table*}

\subsubsection{\textbf{Effect of random sampling}}
An interesting observation can be made from Table \ref{tab:mnistbig} for the MNIST dataset: the error of CPCA is also lower than FRPCAG. Such cases can also be observed in USPS dataset (Table \ref{tab:resultsUSPS}). {As the downsampling step is random, it might remove some spurious samples sometimes and  the clustering scheme (Section \ref{sec:clustering}) becomes robust to these samples.  For the clustering application, the spurious samples mostly lie on the cluster borders and deteriorate the clustering decision. For the computational purposes (validation for all the noise scenarios) Tables \ref{tab:resultsUSPS} \& \ref{tab:mnistbig} correspond to one run of the CPCA for one specific sampling case.}

{In order to study the effect of random sampling on the clustering performance, we perform an experiment on the full USPS dataset ($256 \times 10000$) with different levels of artificial noise. The results in Fig. \ref{fig:sampling_effect} correspond to 10 runs of the CPCA under different uniform sampling scenarios. For this experiment, we downsample 10 times along the columns (digits), whereas no downsampling is used across the features and then add different levels of Gaussian noise from 0 to 15\% in the dataset. Thus, we downsample the dataset, run FRPCAG to get a low-rank $\tilde{X} \in \Rbb^{256 \times 1000}$, perform $k$-means ($k = 10$) and then use the approximate clustering decoder (Algorithm 3) to decode the labels for the full dataset. This process is repeated over the whole parameter grid $\gamma^{'}_r, \gamma^{'}_c \in (0,30)$ and the minimum error over the grid is considered. Each of the boxplots in this figure summarize the clustering error over 10 runs of CPCA(10,1) for different levels of Gaussian noise. The mean clustering error is 15.05\% for the case of no noise and 15.2\%, 15.6\% and 16\% for 5\%, 10\% and 15\% Gaussian noise respectively. Furthermore, the standard deviation for each of the boxplots varies between 0.4\% to 0.55\%. This result clearly shows that the CPCA performance is quit robust to random sampling. Similar results were observed for other datasets and are not reported here for the purpose of brevity.}

%\subsubsection{\textbf{Computation Time for Clustering}}

{It is interesting to study the reduction in the total time attained by using CPCA as compared to FRPCAG. Table \ref{tab:time} can be used to perform this comparison as well. For example, for the MNIST dataset with 5000 samples, the total time (including graph construction) for FRPCAG is 28.2 secs and that for CPCA is 20.1 secs. Thus, a speed-up of 1.4 times is obtained over FRPCAG. The time required for the construction of the graph between the samples or features is often more than that required for the CPCA to converge. This is a small computational bottleneck of the graph construction algorithm.  While graph learning or graph construction is an active and interesting field of research, it is not a part of this work. The state-of-the-art results \cite{kalofolias2016learn}, \cite{pavez2016generalized}, \cite{segarra2016network} do not provide any scalable solutions yet.  }

\subsubsection{\textbf{Rank Preservation Under Downsampling}}
An interesting question about CPCA is if it preserves the underlying rank of the dataset under the proposed sampling scheme. Table \ref{tab:rank} shows that the rank is preserved even in the presence of noise.  For this experiment, we take different datasets and corrupt them with different types of noise and perform cross-validation for clustering using the parameter range for CPCA mentioned in Table \ref{tab:models_param} (see Appendices). Then,  we report the rank of  $\tilde{X}$ for the parameter corresponding to the minimum clustering error. As $\tilde{X}$ is approximately low-rank so we use the following strategy to determine the rank $k$: $\tilde{\Sigma}_{k,k}/\tilde{\Sigma}_{1,1} < 0.1$. FRPCAG assumes that the number of clusters $\approx$ rank of the dataset. Our findings show that this assumption is almost satisfied for the sampled matrices even in the presence of various types of noise. Thus, the rank is preserved under the proposed sampling strategy. For clustering experiments, the lowest error with CPCA occurs when the rank $\approx$ number of clusters.

\begin{table}[htbp]
\caption{Variation of clustering error of CPCA with different uniform downsampling schemes / factors across rows and columns of the USPS dataset ($256 \times 10,000$).}
\centering
\resizebox{0.4\textwidth}{!}{\begin{tabular}[t]{| c | c | c | c | c | c |}
\hline
\textbf{downsampling } &     &    &    &     &    \\\cline{2-6}
\textbf{(rows / cols)} & \textbf{1}  & \textbf{5}  & \textbf{10} & \textbf{15}  & \textbf{20} \\\hline
1   &  0.16  & 0.16 &  0.21      &   0.21   & 0.21 \\\hline
2 &   0.21  & 0.23   &   0.23       & 0.24    &  0.25 \\\hline
4  &   0.26 &  0.30 & 0.31     & 0.31   &  0.31 \\\hline
\end{tabular}}
\label{tab:clustering_cpca}
\end{table}

\subsubsection{\textbf{Clustering error vs downsampling rate}}
Table \ref{tab:clustering_cpca} shows the variation of clustering error of CPCA with different downsampling factors across rows and columns of the USPS dataset ($256 \times 10,000$). Obviously, higher downsampling results in an increase in the clustering error. However, note that we can downsample the samples (columns) by a factor of 5 without observing an error increase. The downsampling of features results in an error increase because the number of features for this dataset is only 256 and downsampling results in a loss of information. Similar behavior can also be observed for the ORL small and ORL large datasets in Table \ref{tab:resultsORL} where the performance of CPCA is slightly worse than FRPCAG because the number of samples $n$ for ORL is only 400.
%\vspace{-0.3cm}

\begin{figure*}[ht]
\includegraphics[width=1.0\textwidth]{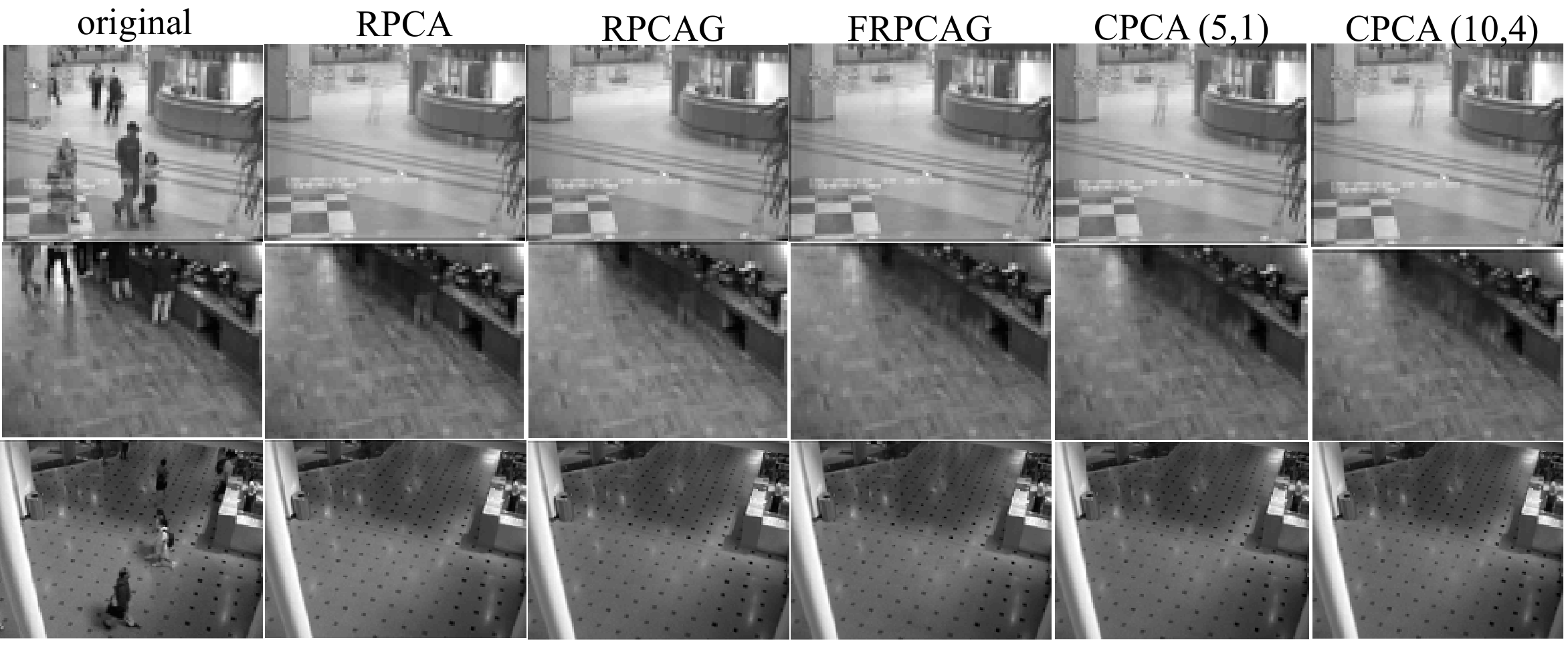}
\caption{Static background separation from videos using different PCA based models. The first row corresponds to a frame from the video of a shopping mall lobby, the second row to a restaurant food counter  and the third row to an airport lobby. The leftmost plot in each row shows the actual frame, the other 5 show the recovered low-rank  using RPCA, RPCAG, FRPCAG and CPCA with two different uniform downsampling schemes.}
\label{fig:videos}
\end{figure*}
\begin{figure*}
\includegraphics[width=1.0\textwidth]{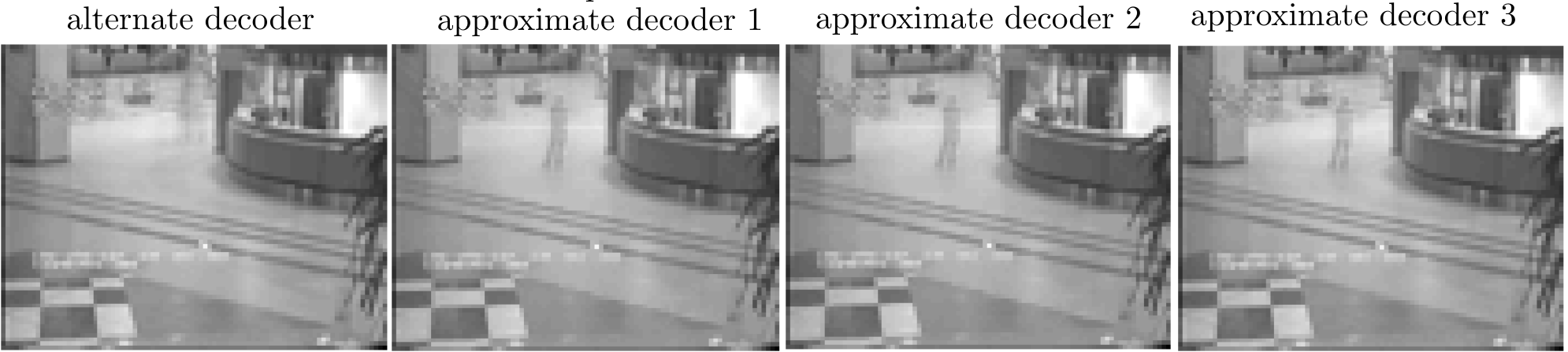}
\caption{A quality comparison of various low-rank decoders discussed in this work.}
\label{fig:decoder_comparison} 
\end{figure*}

\subsection{\textbf{Low-rank recovery}}
%\subsubsection{\textbf{Background Separation from Videos}}
In order to demonstrate the effectiveness of our model to recover low-rank static background from videos we perform experiments on 1000 frames of 3 videos available online \footnotemark[7]\footnotetext[7]{\url{https://sites.google.com/site/backgroundsubtraction/test-sequences}}. All the frames are vectorized and arranged in a data matrix $Y$ whose columns correspond to frames. The graph $G_{c}$ is constructed between the columns of $Y$ and the graph $G_{r}$ is constructed between the rows of $Y$ following the methodology of Section~\ref{sec:graphs}. Fig. \ref{fig:videos} shows the recovery of low-rank  frames for one actual frame of each of the videos. The first row corresponds to a frame from the video of a shopping mall lobby, the second row to a restaurant food counter  and the third row to an airport lobby. The leftmost plot in each row shows the actual frame, the other 5 show the recovered low-rank representations using RPCA, RPCAG, FRPCAG and CPCA with two different uniform downsampling rates. For CPCA Algorithm 2 is used and the rank $k$ for the approximate decoder is set such that $\tilde{\Sigma}_{k,k}/\tilde{\Sigma}_{1,1} < 0.1$, where $\tilde{\Sigma}$ are the singular values of $\tilde{X}$.

For the 2nd and 3rd rows of Fig. \ref{fig:videos} it can be seen that our proposed model is able to separate the static backgrounds very accurately from the moving people which do not belong to the static ground truth.  However, the quality is slightly compromised in the 1st row where the shadow  of the person appears in the low-rank frames recovered with CPCA. In fact, this person remains static for a long time in the video and the uniform sampling compromises the quality slightly. 

\begin{table}[htbp]
\caption{Computational time in seconds of RPCA, RPCAG, FRPCAG and CPCA for low-rank recovery of different videos in Fig. \ref{fig:videos}. }
\centering
\resizebox{0.5\textwidth}{!}{\begin{tabular}[t]{| c | c | c | c | c | c |}
\hline
\textbf{Videos} &  RPCA & RPCAG &  FRPCAG  & CPCA (5,1)  & CPCA (10,4) \\\hline
1   &  2700 & 3550  & 120  & 21  & 8   \\\hline
2   &  1650  & 2130 & 85   & 15  &  6  \\\hline
3   & 3650  &  4100 &  152  &  32  & 11  \\\hline
\end{tabular}}
\label{tab:video_times}
\end{table}

%\subsubsection{\textbf{Computation Time for Background Separation}}
Table \ref{tab:video_times} presents the computational time in seconds of RPCA, RPCAG, FRPCAG and CPCA for low-rank recovery of different videos in Fig. \ref{fig:videos}. The time reported here corresponds to steps 2 to 6 of Table \ref{tab:cpca}, Algorithm 1 of \cite{shahid2015fast} for FRPCAG, \cite{candes2011robust} for RPCA and  \cite{shahid2015robust} for RPCAG, excluding the construction of graphs $G_r, G_c$. The speed-up observed for these experiments from Table \ref{tab:video_times} is 10 times over FRPCAG and 100 times over RPCA and RPCAG. 

%\subsubsection{\textbf{Quality Comparison of Decoders}}
Fig. \ref{fig:decoder_comparison}  presents a comparison of the quality of the low-rank static background extracted using the alternate \eqref{eq:alternatedecoder} and approximate decoders discussed in \eqref{eq:Uupsample} for a video (1st row) of Fig. \ref{fig:videos}. Clearly, the alternate decoder performs slightly better than the approximate decoders but at the price of tuning of two model parameters.

%\subsubsection{\textbf{Quality of Low-Rank with Downsampling}}
Fig. \ref{fig:videos_sampling}  presents a comparison of the quality of low-rankness for the same video extracted using the approximate decoder \eqref{eq:Uupsample} using different downsampling factors on the pixels and frames. It is obvious that the quality of low-rankness remains intact even with higher downsampling factors.
%\vspace{-0.4cm}

\begin{figure*}
\centering
\includegraphics[width=0.9\textwidth]{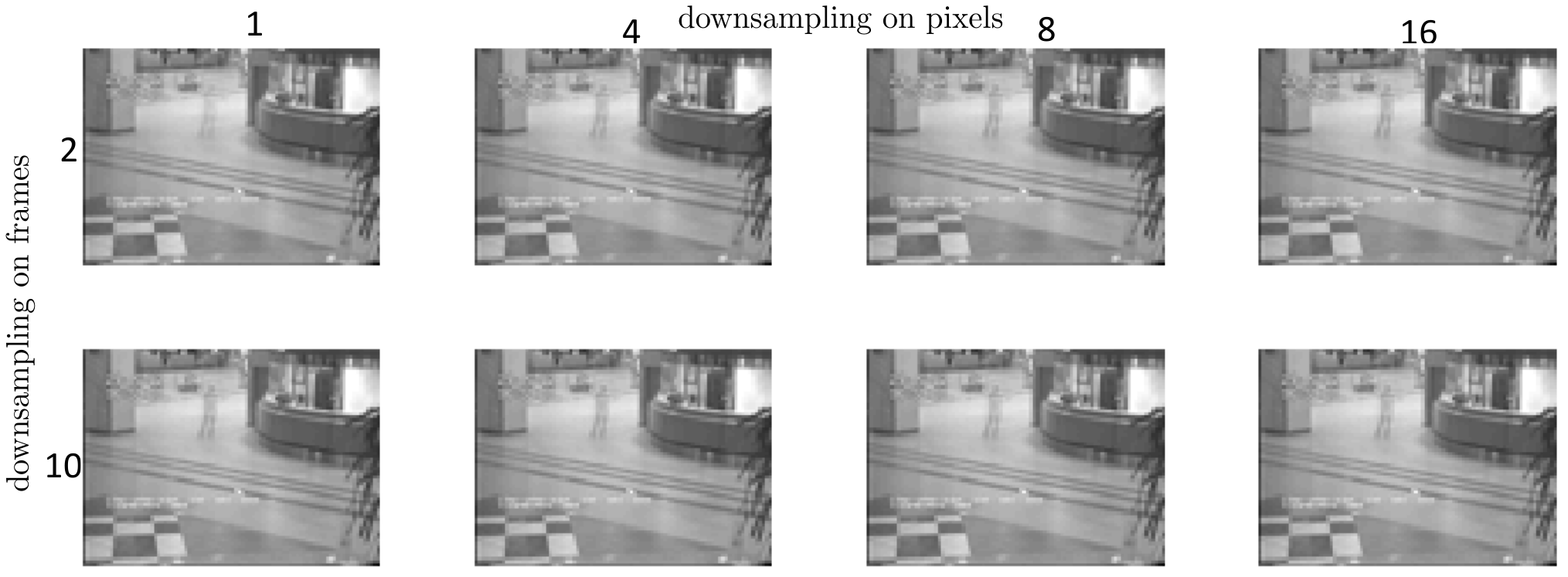}
\caption{A comparison of the quality of low-rank frames for the shopping mall video (1st row of Fig. \ref{fig:videos}) extracted using the approximate decoder \eqref{eq:Uupsample} for different downsampling factors on the pixels and frames. It is obvious that the quality of low-rank remains intact even with higher downsampling factors.}
\label{fig:videos_sampling}
\end{figure*} 

\subsection{\textbf{Low-Rank Recovery from Random Projections}}
{Throughout this work we assume that the graphs $\G_c$ and $\G_r$ for the complete data matrix $Y$ are either available or can be constructed directly from $Y$ itself using the standard graph construction algorithms. This is a reasonable assumption if one wants to reduce the computational burden by downsampling on the datasets. However, often the complete data matrix $Y$ is not available and the goal is to obtain an estimate of $Y$ from some side information. Typical examples include Magnetic Resonance Imaging (MRI), Computational Tomography (CT) and Electron Tomography (ET) where one  only has access to the projections $b$ of $Y$ acquired through a known projection operator $A$. The purpose here is not to reduce the computational burden but to acquire a good enough estimate of $Y$ from $b$. Furthermore, for such applications, there is no notion of row or column projection / sampling operators. Nevertheless, one might want to exploit the row and column smoothness assumption for the purpose of reconstruction. While, this is not a significant part of our current work, it is still an obvious open question and the answer comes from an extension of this work. Therefore, to be complete, we propose a framework for such problems which might require a low-rank reconstruction from a few projections. It is important to emphasize though that the goal is not to compare and evaluate the performance rigorously with the state-of-the-art. In fact we mention this here just to give a flavour of how the current framework can be extended for such problems. }  

{Assume that a CT sample, for example, a Shepp-Logan phantom of the size $X \in \Rbb^{p \times n}$ needs to be reconstructed from its projections $b \in \Rbb^{m}$, obtained via a line projection matrix $A \in \Rbb^{m \times np}$. Thus, $b = A vec(X) + e$, where $e \in \Rbb^{m}$ models the noise in the projections. We propose to reconstruct the sample $X$ by solving the following optimization problem:}
\begin{equation}\label{eq:adaptive_rec}
\min_{X} \|A vec(X) - b\|^2_2 + \gamma_r \tr(X^\top \Larg_r X) + \gamma_c \tr(X \Larg_c X^\top),
\end{equation}
{where $\Larg_r, \Larg_c$ are the row and column graph Laplacians between the rows and columns of $X$. Since, these graphs are not available in the beginning, one can obtain an initial estimate of $X$ by running a standard compressed sensing problem for a few iterations and then construct these graphs from this estimate. The estimated graphs can also be improved after every few iterations from the more refined $X$.}

\begin{figure}
\centering
\includegraphics[width=0.4\textwidth]{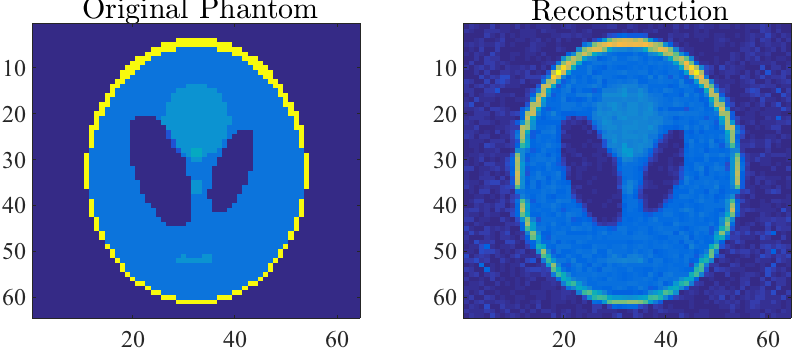}
\caption{The reconstruction of a $64\times 64$ Modified Shepp-Logan phantom from 20\% projections using eq.\eqref{eq:adaptive_rec}.}
\label{fig:phantom}
\end{figure} 

{Fig. \ref{fig:phantom} shows the reconstruction of a $64\times 64$ Modified Shepp-Logan phantom from 20\% projections using eq.\eqref{eq:adaptive_rec}. The initial estimate of the graphs $\G_r, \G_c$ between the rows and columns of the phantom is obtained from the first 3 iterations of the compressed sensing based recovery problem and then these graphs are updated every 5 iterations. Our future work will focus on a detailed study of this method.}

%\vspace{-0.2cm}
\section{\textbf{Conclusion}}
We present Compressive PCA on Graphs (CPCA) which approximates a recovery of low-rank matrices on graphs from their sampled measurements. It is supported by the proposed restricted isometry property (RIP) which  is related to the coherence of the eigenvectors of graphs between the rows and columns of the data matrix.  Accompanied with several efficient, parallel, parameter free and low-cost decoders for low-rank recovery and clustering, the presented framework gains a several orders of magnitude speed-up over the low-rank recovery methods like Robust PCA. Our theoretical analysis reveals that CPCA targets exact recovery for low-rank matrices which are clusterable across the rows and columns. Thus, the error depends on the spectral gaps of the graph Laplacians. Extensive clustering experiments on 5 datasets with various types of noise and comparison with 11 state-of-the-art methods reveal the efficiency of our model. CPCA also achieves state-of-the-art results for background separation from videos.

% \clearpage

 \bibliographystyle{ieee}
\bibliography{pcabib}

\onecolumn
\clearpage
\newpage
\appendix

\section{\textbf{Appendices}}\label{sec:A}
\subsection{Proof of theorem \ref{thm:embedding}}\label{sec:embedding_proof}
%\begin{proof}
We start with the sampling of the rows. Theorem $5$ in \cite{puy2015random} shows that for any $\delta_r, \epsilon_r \in (0, 1)$, with probability at least $1 - \epsilon_r$, 
\begin{align*}
(1-\delta_r) \norm{z}_2^2 \leq \frac{p}{\rho_r} \norm{M_r z}_2^2 \leq (1+\delta_r) \norm{z}_2^2
\end{align*}
for all $z \in {span}(P_{k_r})$ provided that
\begin{align}
\label{eq:cond1}
\rho_r \geq \frac{3}{\delta_r^2} \nu_{{k_r}}^2 \log\left(\frac{2k_r}{\epsilon_r}\right).
\end{align}
Notice that Theorem $5$ in \cite{puy2015random} is a uniform result. As a consequence, with probability at least $1 - \epsilon_r$, 
\begin{align}\label{eq:ripr}
(1-\delta_r) \norm{y_i}_2^2 \leq \frac{p}{\rho_r} \norm{M_r y_i}_2^2 \leq (1+\delta_r) \norm{y_i}_2^2, \quad i = 1, \ldots, n,
\end{align}
for all $y_1, \ldots, y_n \in {span}(P_{k_r})$ provided that \eqref{eq:cond1} holds. Summing the previous inequalities over all $i$ shows that, with probability at least $1 - \epsilon_r$, 
\begin{align}
\label{eq:RIP_row}
(1-\delta_r) \norm{Y}_F^2 \leq \frac{p}{\rho_r} \norm{M_r Y}_F^2 \leq (1+\delta_r) \norm{Y}_F^2,
\end{align}
for all $Y \in \Re^{p \times n}$ with column-vectors in ${span}(P_{k_r})$. 

Let us continue with the sampling of the columns. Again, Theorem $5$ in \cite{puy2015random} shows that for any $\delta_c, \epsilon_c \in (0, 1)$, with probability at least $1 - \epsilon_c$, 
\begin{align*}
(1-\delta_c) \norm{w}_2^2 \leq \frac{n}{\rho_c} \norm{w^\top M_c}_2^2 \leq (1+\delta_c) \norm{w}_2^2
\end{align*}
for all $w \in {span}(Q_{k_c})$ provided that
\begin{align}
\label{eq:cond2}
\rho_c \geq \frac{3}{\delta_c^2} \nu_{{k_c}}^2 \log\left(\frac{2k_c}{\epsilon_c}\right). 
\end{align}
As a consequence, with probability at least $1 - \epsilon_c$, 
\begin{align}\label{eq:ripc}
(1-\delta_c) \norm{z_i}_2^2 \leq \frac{n}{\rho_c} \norm{z_i^\top M_c}_2^2 \leq (1+\delta_c) \norm{z_i}_2^2, \quad i = 1, \ldots, \rho_r,
\end{align}
for all $z_1, \ldots, z_{\rho_r} \in {span}(Q_{k_c})$ provided that \eqref{eq:cond2} holds. Summing the previous inequalities over all $i$ shows that, with probability at least $1 - \epsilon_c$, 
\begin{align}
\label{eq:RIP_col}
(1-\delta_c) \norm{Z}_F^2 \leq \frac{n}{\rho_c} \norm{Z M_c}_F^2 \leq (1+\delta_c) \norm{Z}_F^2
\end{align}
for all $Z \in \Re^{\rho_r \times n}$ with row-vectors in ${span}(Q_{k_c})$. In particular, this property holds, with at least the same probability, for all matrices $Z$ of the form $M_r Y$ where $Y \in \Re^{p \times n}$ is a matrix with row-vectors in ${span}(Q_{k_c})$.

We now continue by combining \eqref{eq:RIP_row} and \eqref{eq:RIP_col}. We obtain that
\begin{align}
\label{eq:RIP_tight}
(1-\delta_c) (1-\delta_r) \norm{Y}_F^2 \leq \frac{np}{\rho_c\rho_r} \norm{M_r Y M_c}_F^2 \leq (1+\delta_c) (1+\delta_r) \norm{Y}_F^2
\end{align}
for all $Y \in \Rbb^{p \times n}$ with column-vectors in ${span}(P_{k_r})$ and row-vectors in ${span}(Q_{k_c})$, provided that \eqref{eq:cond1} and \eqref{eq:cond2} hold. It remains to compute the probability that \eqref{eq:RIP_tight} holds. Property \eqref{eq:RIP_tight} does not hold if \eqref{eq:RIP_row} or \eqref{eq:RIP_col} do not hold. Using the union bound, \eqref{eq:RIP_tight} does not hold with probability at most $\epsilon_r + \epsilon_c$. To finish the proof, one just need to choose $\epsilon_r = \epsilon_c = \epsilon/2$ and $\delta_r = \delta_c = \delta/3$, and notice that $(1 + \delta/3)^2 \leq 1+\delta$ and $(1 - \delta/3)^2 \geq 1 - \delta$ for $\delta \in (0, 1)$.
%\end{proof}

\subsection{{Proof of Theorem~\ref{thm:idealdecoder}}}\label{sec:idealdecoder}
%\begin{proof}
Using the optimality condition we have, for any $Z \in \Rbb^{p \times n}$,
\begin{align*}
\norm{M_r X^* M_c - \tilde{X}}_F \leq \norm{M_r Z M_c - \tilde{X}}_F.
\end{align*}
For $Z = \bar{X}$, we have
\begin{align*}
\norm{M_r X^* M_c - \tilde{X}}_F \leq \norm{M_r \bar{X} M_c - \tilde{X}}_F,
\end{align*}
which gives
\begin{align*}
\norm{M_r X^* M_c - M_r \bar{X} M_c - \tilde{E}}_F \leq \norm{\tilde{E}}_F.
\end{align*}
As \eqref{eq:rip} holds, we have
\begin{align*}
\norm{M_r X^* M_c - M_r \bar{X} M_c - \tilde{E}}_F 
& \geq \norm{M_r (X^* - \bar{X}) M_c}_F - \norm{\tilde{E}}_F\\
& \geq \sqrt{\frac{\rho_r\rho_c(1-\delta)}{n p}} \norm{X^* - \bar{X}}_F - \norm{\tilde{E}}_F.
\end{align*}
Therefore, by combining the above equations we get
\begin{align*}
\norm{X^* - \bar{X}}_F \leq 2 \sqrt{\frac{n p}{\rho_r\rho_c(1-\delta)}} \norm{\tilde{E}}_F.
\end{align*}

%\end{proof}

\subsection{{Proof of Theorem~\ref{thm:alternatedecoder}}}\label{sec:alternatedecoder}

%\begin{proof}
Using the optimality condition we have for any $Z \in \Re^{p \times n}$ and optimal solution ${X^{*}} = \bar{X}^* + E^*$:
\begin{equation}\label{eq:ad1}
\|M_r X^{*} M_c - \tilde{X}\|^{2}_{F} + \bar{\gamma_c}\tr(X^{*}\Larg_c {X^{*}}^{\top}) + \bar{\gamma_r}\tr({X^{*}}^{\top}\Larg_r X^{*}) \leq \|M_r Z M_c - \tilde{X}\|^{2}_{F} + \bar{\gamma_c}\tr(Z\Larg_c {Z}^{\top}) + \bar{\gamma_r}\tr({Z}^{\top}\Larg_r Z)
\end{equation}
using $Z = \bar{X} = P_{k_r}Y_{b}Q^{\top}_{k_c}$ as in the proof of theorem 2 in \cite{shahid2015fast}, where $Y_b \in \Re^{k_r \times k_c}$ and it is not necessarily diagonal. Note that $\|Y_b\|_F = \|\bar{X}\|_F$,  $Q^{\top}_{k_c}Q_{k_c} = I_{k_c}$, $P^{\top}_{k_r}P_{k_r} = I_{k_r}$, $\bar{Q}^{\top}_{k_c}Q_{k_c} = 0$, $\bar{P}^{\top}_{k_r}P_{k_r} = 0$. From the proof of theorem 2 in   \cite{shahid2015fast} we also know that:

$$\tr(\bar{X} \Larg_c (\bar{X})^\top)  \leq  \lambda_{k_c} \|\bar{X}\|_F^2 $$
$$\tr(\bar{X}^\top \Larg_r (\bar{X}))  \leq  \lambda_{k_r} \|\bar{X}\|_F^2 $$
$$\tr(X^* \Larg_c (X^*)^\top)  \geq  \lambda_{k_{c+1}} \|X^* \bar{Q}_{k_c}\|_F^2 $$
$$\tr(X^* \Larg_r (X^*)^\top)  \geq  \lambda_{k_{r+1}} \|\bar{P}_{k_r}^\top X^*\|_F^2 $$

Now using all this information in \eqref{eq:ad1} we get
\begin{align*}
 \|M_r X^{*} M_c - \tilde{X}\|^{2}_{F} +  \bar{\gamma}_c \lambda_{k_{c}+1}\|X^{*}\bar{Q}_{k_c}\|^2_F + \bar{\gamma}_r \lambda_{k_{r}+1}\|\bar{P}_{k_r}^\top X^*\|_F^2 \leq \|\tilde{E}\|^{2}_{F} + (\bar{\gamma}_c  \lambda_{k_c} + \bar{\gamma}_r \lambda_{k_r})\|\bar{X}\|^{2}_{F}
\end{align*}

 From above we have:
\begin{equation}
\|M_r X^{*} M_c - \tilde{X}\|_{F} \leq \|\tilde{E}\|_{F} + \sqrt{(\bar{\gamma}_c \lambda_{k_c} + \bar{\gamma}_r \lambda_{k_r})}\|\bar{X}\|_{F}
\end{equation}
and 
\begin{equation}
\sqrt{(\bar{\gamma}_c \lambda_{k_{c}+1}\|X^{*}\bar{Q}_{k_c}\|^2_F + \bar{\gamma}_r \lambda_{k_{r}+1}\|\bar{P}_{k_r}^\top X^*\|_F^2)} \leq \|\tilde{E}\|_{F} + \sqrt{(\bar{\gamma}_c \lambda_{k_c} + \bar{\gamma}_r  \lambda_{k_r})}\|\bar{X}\|_{F}
\end{equation}
using $$\bar{\gamma}_c = \gamma \frac{1}{\lambda_{k_{c}+1}} \quad \text{and} \quad \bar{\gamma}_r = \gamma \frac{1}{\lambda_{k_{r}+1}},$$ and $$\|E^*\|^2_F = \|X^{*}\bar{Q}_{k_c}\|^2_F = \|\bar{P}_{k_r}^\top X^*\|_F^2$$  we get:
\begin{equation}\label{eq:i1}
\|M_r X^{*} M_c - \tilde{X}\|_{F} \leq \|\tilde{E}\|_{F} + \sqrt{\gamma\Big(\frac{\lambda_{k_c}}{\lambda_{k_c + 1}} +  \frac{\lambda_{k_r}}{\lambda_{k_r + 1}}\Big)}\|\bar{X}\|_{F}
\end{equation}
and 
\begin{equation}\label{eq:i2}
\sqrt{2\gamma}\|E^*\|_F \leq \|\tilde{E}\|_{F} + \sqrt{\gamma\Big(\frac{\lambda_{k_c}}{\lambda_{k_c + 1}} +  \frac{\lambda_{k_r}}{\lambda_{k_r + 1}}\Big)}\|\bar{X}\|_{F}
\end{equation}
which implies
\begin{equation}
\|E^*\|_F \leq \frac{\|\tilde{E}\|_{F}}{\sqrt{2\gamma}} + \frac{1}{\sqrt{2}}\sqrt{\gamma\Big(\frac{\lambda_{k_c}}{\lambda_{k_c + 1}} +  \frac{\lambda_{k_r}}{\lambda_{k_r + 1}}\Big)}\|\bar{X}\|_{F}
\end{equation}
Focus on $\|M_r X^{*} M_c - \tilde{X}\|^{2}_{F}$ now. As $M_r, M_c$ are constructed with a sampling without replacement, we have $\|M_r E^* M_c\|_F \leq \|E^*\|_F$. Now using the above facts and the RIP we get:
\begin{align*}
\|M_r X^{*} M_c - \tilde{X}\|_{F}  & = \|M_r(\bar{X}^{*}+E^{*})M_c - M_r \bar{X}M_c - \tilde{E}\|_{F} \\
& \geq \sqrt{\frac{\rho_r\rho_c(1-\delta)}{np}}\|\bar{X}^{*}-\bar{X}\|_{F}  - \|\tilde{E}\|_{F} - \|E^{*}\|_{F}
\end{align*}
this implies
\begin{align*}
\|\bar{X}^{*}-\bar{X}\|_{F} & \leq \sqrt{\frac{np}{\rho_c\rho_r(1-\delta)}}\Bigg[\Big(2 + \frac{1}{\sqrt{2\gamma}}\Big) \|\tilde{E}\|_F + (\frac{1}{\sqrt{2}} + \sqrt{\gamma})\sqrt{\Big(\frac{\lambda_{k_c}}{\lambda_{k_c + 1}} +  \frac{\lambda_{k_r}}{\lambda_{k_r + 1}}\Big)}\|\bar{X}\|_F\Bigg]
\end{align*}
%\end{proof}

\textbf{Discussion}
Let $A_1, A_2 \in \Re^{p \times n}$ and $A_1=U_1S_1V_1^T$, $A_2=U_2S_2V_2^T$ then if $\|A_1 - A_2\|^2_{F} \rightarrow 0$, then $S_1 \rightarrow S_2$.

We observe that
$$
   \|A_1-A_2\|^2_{F} = \|U_1 S_1 V_1^T - U_2 S_2 V_2^T\|^2_{F} = \|U_2^T U_1 S_1 V_1^T V_2 - S_2\|^2_{F}
$$
which implies that $U_2^T U_1 S_1 V_1^T V_2 \approx S_2 $. This is equivalent to saying that for the significant values of $S_2$, the  orthonormal matrices $U_2^T U_1$ and $V_1^T V_2$ have to be almost diagonal. As a result, for the significant values of $S_2$, $U_2$ and $V_2$ have to be aligned with $U_1$ and $V_1$. The same reason also implies that $S_1\approx S_2$.

 \subsection{{Solution of eq.~\eqref{eq:UUI}}}\label{sec:solUUI}

Let us examine how to solve \eqref{eq:UUI}. The problem can be reformulated as:
\begin{align*}
& \min_{U}   \tr(U^{\top} \Larg_r U) \quad \text{s.t:} \quad U^{\top}U = I_k, ~ \|M_r U - \tilde{U}\|^{2}_{F} < \epsilon
\end{align*}
Let $U^{'}$ is the zero appended matrix of $\tilde{U}$, then we can re-write it as:
\begin{align*}
& \min_{U}   \tr(U^{\top} \Larg_r U) \quad \text{s.t:} \quad U^{\top}U = I_k, ~ \|M_r( U - {U}^{'})\|^{2}_{F} < \epsilon
\end{align*}
The above problem is equivalent to \eqref{eq:UUI}, as the term $\|M_r( U - {U}^{'})\|^{2}_{F}$ has been removed from the objective and introduced as a constraint. Note that the constant $\gamma_r$ is not needed anymore. The new model parameter $\epsilon$ controls the radius of the $L_2$ ball   $\|M_r (U - {U}^{'})\|^{2}_{F}$. In simple words it controls how much noise is tolerated by the projection of $U$ on the ball that is centered at $U^{'}$. To solve the above problem one needs to split it down into two sub-problems and solve iteratively between:
\begin{enumerate}
\item The optimization  $\min_{U}   \tr(U^{\top} \Larg_r U) \quad \text{s.t:} \quad U^{\top}U = I_k $.
The solution to this  problem is given by the lowest $k$ eigenvectors of $\Larg_r$. Thus it  requires a complexity of $\mathcal{O}((n+p)k^2)$ for solving both problems \eqref{eq:UUI}.
\item The projection on the $L_2$ ball $\|M_r( U - {U}^{'})\|^{2}_{F}$ whose complexity is $\mathcal{O}(\rho_c + \rho_r)$. 
\end{enumerate}
Thus the solution requires a double iteration with a complexity of $\mathcal{O}(Ink^2)$ and is almost as expensive as FRPCAG.

\subsection{{Proof of Theorem~\ref{thm:approxdecoder}}}\label{sec:approxdecoder}

%\begin{proof}
We can write~\eqref{eq:u} and~\eqref{eq:v} as following:
\begin{align}\label{eq:ui}
\min_{u_1 \cdots u_p} & \sum_{i=1}^p \big[\|M_r u_i - \tilde{u}_i\|^{2}_{2} + {\gamma^{'}}_r u^{\top}_{i} \Larg_r u_i \big]
\end{align}
\begin{align}\label{eq:vi}
\min_{v_1 \cdots v_n} & \sum_{i=1}^n \big[\| M^{\top}_c v_i - \tilde{v}_i\|^{2}_{2} + {\gamma^{'}}_c v^{\top}_{i} \Larg_c v_i \big]
\end{align}
In this proof, we only treat Problem \eqref{eq:ui} and the recovery of $\bar{U}$. The proof for Problem \eqref{eq:v} and the recovery of $\bar{V}$ is identical. The above two problems can be solved independently for every $i$. From theorem 3.2 of \cite{puy2015random} we obtain:
\begin{align}\label{eq:u_1}
\norm{\bar{u}_i^* - \bar{u}_i}_2 
\leq 
\sqrt{\frac{p}{\rho_r(1-\delta)}} 
\left[ \left( 2 + \frac{1}{\sqrt{\gamma_r' \lambda_{k_r+1}}} \right) \norm{\tilde{e}_i^u}_2 
+ \left( \sqrt{\frac{\lambda_{k_r}}{\lambda_{k_r+1}}} + \sqrt{\gamma_r' \lambda_{k_r}} \right) \norm{\bar{u}_i}_2 \right],
\end{align}
and
\begin{align*}
\norm{e_i^*}_2
\leq 
\frac{1}{\sqrt{\gamma_r' \lambda_{k_r+1}}} \norm{\tilde{e}_i^u}_2 
+ \sqrt{\frac{\lambda_{k_r}}{\lambda_{k_r+1}}} \norm{\bar{u}_i}_2,
\end{align*}

which implies
\begin{align*}
\norm{\bar{u}_i^* - \bar{u}_i}_2^2 
\leq 
2 \frac{p}{\rho_r(1-\delta)}
\left[ \left( 2 + \frac{1}{\sqrt{\gamma_r' \lambda_{k_r+1}}} \right)^2 \norm{\tilde{e}_i^u}_2^2 
+ \left( \sqrt{\frac{\lambda_{k_r}}{\lambda_{k_r+1}}} + \sqrt{\gamma_r' \lambda_{k_r}} \right)^2 \norm{\bar{u}_i}_2^2 \right],
\end{align*}
and
\begin{align*}
\norm{e_i^*}_2^2
\leq 
\frac{2}{\gamma_r' \lambda_{k_r+1}} \norm{\tilde{e}_i^u}_2^2 
+ 2\frac{\lambda_{k_r}}{\lambda_{k_r+1}} \norm{\bar{u}_i}_2^2.
\end{align*}
Summing the previous inequalities over all $i$'s yields
\begin{align*}
\norm{\bar{U}^* - \bar{U}}_F^2 
\leq 
2 \frac{p}{\rho_r(1-\delta)}
\left[ \left( 2 + \frac{1}{\sqrt{\gamma_r' \lambda_{k_r+1}}} \right)^2 \norm{\tilde{E}^u}_F^2 
+ \left( \sqrt{\frac{\lambda_{k_r}}{\lambda_{k_r+1}}} + \sqrt{\gamma_r' \lambda_{k_r}} \right)^2 \norm{\bar{U}}_F^2 \right],
\end{align*}
and
\begin{align*}
\norm{E^*}_F^2
\leq 
\frac{2}{\gamma_r' \lambda_{k_r+1}} \norm{\tilde{E}^u}_F^2 
+ 2\frac{\lambda_{k_r}}{\lambda_{k_r+1}} \norm{\bar{U}}_F^2.
\end{align*}
Taking the square root of both inequalities terminates the proof. Similarly, the expressions for $\bar{V}$ can be derived:
\begin{align*}
\norm{\bar{V}^* - \bar{V}}_F &
\leq 
\sqrt{\frac{2 n}{\rho_c(1-\delta)}}
\Bigg[ \left( 2 + \frac{1}{\sqrt{\gamma_c' \lambda_{k_c+1}}} \right) \norm{\tilde{E}^v}_F 
+ \left( \sqrt{\frac{\lambda_{k_c}}{\lambda_{k_c+1}}} + \sqrt{\gamma_c' \lambda_{k_c}} \right) \norm{\bar{V}}_F \Bigg]
\end{align*}
and 
\begin{align*}
\norm{E^*}_F
\leq 
\sqrt{\frac{2}{\gamma_c' \lambda_{k_c+1}}} \norm{\tilde{E}^v}_F
+ \sqrt{2\frac{\lambda_{k_c}}{\lambda_{k_c+1}}} \norm{\bar{V}}_F.
\end{align*}

\subsection{{Proof of Lemma~\ref{lemma:graphupsampling}}}\label{sec:graphupsampling}
%\begin{proof}
Let $S = [{S}^{\top}_{a} | S^{\top}_{b}]^{\top}$. Further we split $\Larg$ into submatrices as follows:

$$
\Larg =  \left[ \begin{array}{cc}
\Larg_{aa}  & \Larg_{ab} \\
\Larg_{ba}  & \Larg_{bb}
 \end{array} \right] 
$$

Now~\eqref{eq:graphupsampling} can be written as:
\begin{align*}
& \min_{S_a} 
 {\left[ \begin{array}{c}
S_a \\
S_b 
\end{array} \right]}^{\top}
\left[ \begin{array}{cc}
\Larg_{aa}  & \Larg_{ab} \\
\Larg_{ba}  & \Larg_{bb}
 \end{array} \right] 
 {\left[ \begin{array}{c}
S_a \\
S_b 
\end{array} \right]} \nonumber\\
& \text{s.t:} \hspace{0.2cm} S_b = R
\end{align*}
further expanding we get:
$$
\min_{S_a} S^{\top}_a\Larg_{aa}S_a + S^{\top}_a\Larg_{ab}R + R^{\top}\Larg_{ba}S_a + R\Larg_{bb}R
$$
using $\nabla S_a = 0$, we get:
$$
2\Larg_{ab}R + 2\Larg_{aa}S_a = 0
$$
$$ S_a = -\Larg^{-1}_{aa}\Larg_{ab}R $$
%The inverse above is well defined if the $rank(\Larg) = c - 1 $. \nati{Are you sure? What is $c$? I believe that the inverse of $\Larg_aa$ is well defined if you remove at least one node per connected component.} \nauman{I dont know why I wrote this :P.Lets remove this last point}. This ensures that $rank(\Larg_{aa}) = d$.
%\end{proof}
 
 \subsection{Other Approximate Decoders}\label{sec:approx_decoders}
 Alternatively, if the complete data matrix $Y$ is available then we can reduce the complexity further by performing a graph-upsampling for only one of the two subspaces $U$ or $V$. 
 \subsubsection{Approximate decoder 2 }\label{sec:approxdecoder2}
Suppose we do the upsampling only for $U$, then the approximate decoder 2 can be written as:
\begin{align*}
& \min_{U} \tr(U^{\top}\Larg_r U)  \quad \text{s.t:} \quad M_r U = \tilde{U}.
\end{align*}

The solution for $U$ is given by eq.~\ref{eq:Uupsample}. Then, we can write $V$ as:
$$ V =  Y^\top U \tilde{\Sigma}^{-1}\sqrt{\frac{\rho_c \rho_r (1-\delta)}{np}}$$
However, we do not need to explicitly determine $V$ here. Instead the low-rank $X$ can be determined directly from $U$ with the projection given below:
\begin{align*} 
X = & U \tilde{\Sigma} \sqrt{\frac{np}{\rho_c \rho_r (1-\delta)}}V^{\top} = UU^{\top}Y. 
\end{align*}
 
\subsubsection{Approximate decoder 3 }\label{sec:approxdecoder3}

Similar to the approximate decoder 2, we can propose another approximate decoder 3 which performs a graph upsampling on $V$ and then determines $U$ via matrix multiplication operation.
\begin{align*}
& \min_{V} \tr(V^{\top}\Larg_c V) \quad \text{s.t:} \quad M^{\top}_c V = \tilde{V}
\end{align*}
 The solution for $V$ is given by eq.~\ref{eq:Uupsample}. Using the similar trick as for the approximate decoder 2, we can compute $X$ without computing $U$. Therefore,  $X = Y V V^{\top}$.
 
 For the proposed approximate decoders, we would need to do one SVD to determine the singular values $\tilde(\Sigma)$. However, note that this SVD is on the compressed matrix $\tilde{X} \in \Re^{\rho_{r} \times \rho_c }$. Thus, it is inexpensive $\mathcal{O}(\rho^{2}_{r} \rho_c)$ assuming that $\rho_r < \rho_c$.

\subsection{Computational Complexities \& Additional Results}\label{sec:summary_decoders}
We present the computational complexity of all the models considered in this work. For a matrix $X \in \Re^{p \times n}$, let $I$ denote the number of iterations for the algorithms to converge, $p$ is the number of features, $n$ is the number of samples, $\rho_r, \rho_c$ are the number of features and samples for the compressed data $\tilde{Y}$ and satisfy eq. \eqref{eq:Mc} and theorem \ref{thm:embedding}, $k$ is the rank of the low-dimensional space or  the number of clusters, $\mathcal{K}$ is the number of nearest neighbors for graph construction, $O_l, O_c$ correspond to the number of iterations in the Lancoz and Chebyshev approximation methods.  All the models which use the graph $G_c$ are marked by '+'. The construction of graph $G_r$ is included only in FRPCAG and CPCA. Furthermore,
\begin{enumerate}
\item We assume that $\K, k, \rho_r, \rho_c, p << n$ and $n + p + k + \K + \rho_r + \rho_c \approx n$.
\item The complexity of $\|Y-X\|_1$ is $\mathcal{O}(np)$ per iteration and that of $\|\tilde{Y}-\tilde{X}\|_1$ is $\mathcal{O}(\rho_c \rho_r)$.
\item The complexity of the computations corresponding to the graph regularization $\tr(X\Larg_c X^\top) + \tr(X^\top \Larg_r X) = \mathcal{O}(p |\E_c| + n |\E_r|) =  \mathcal{O}(pn \K + np\K)$, where $\E_r, \E_c$ denote the number of non-zeros in $\Larg_r, \Larg_c$ respectively. Note that we use the $\K$-nearest neighbors graphs so $\E_r \approx \K p$ and $\E_c \approx \K n$. 
\item The complexity for the construction of $\tilde{\Larg}_c$ and $\tilde{\Larg}_r$ for compressed data $\tilde{Y}$ is negligible if FLANN is used, i.e, $\mathcal{O}(\rho_c \rho_r \log (\rho_c))$ and $\mathcal{O}(\rho_c \rho_r \log (\rho_r))$. However, if the kron reduction strategy of Section \ref{sec:small_graphs} is used then the cost is $\mathcal{O}(\K O_l (n +  p)) \approx \mathcal{O}(\K O_l n)$. 
\item We use the complexity $\mathcal{O}(n p^2)$ for all the SVD computations on the matrix $X \in \Re^{p \times n}$ and $\mathcal{O}(\rho_c \rho^2_r)$ for $\tilde{X} \in \Re^{\rho_c \times \rho_r}$.
\item The complexity of $\|M_r X M_c-\tilde{X}\|^2_F$ is negligible as compared to the graph regularization terms $\tr(X\Larg_c X^\top) + \tr(X^\top \Larg_r X)$.
\item We use the approximate  decoders  for low-rank recovery  in the complexity calculations (eq. \eqref{eq:Uupsample} in Section \ref{sec:approx_decoder}). All the decoders for low-rank recovery are summarized in Table \ref{tab:decoders}.
\item The complexity of k-means \cite{elkan2003using} is $\mathcal{O}(Inkp)$ for a matrix $X \in \Re^{p \times n}$ and $\mathcal{O}(I\rho_r \rho_ck)$ for a matrix $\tilde{X}\in \Re^{\rho_r \times \rho_c}$.
\end{enumerate}

\clearpage
\newpage
\twocolumn
\begin{table}[htbp]
\caption{A summary and computational complexities of all the decoders proposed in this work. The Lancoz method used here is presented in \cite{susnjara2015accelerated}.}
\centering
\resizebox{0.55\textwidth}{!}{\begin{tabular}[t]{| c | c | c | c | c | }
\hline
 \textbf{Type}  & \multicolumn{4}{|c|}{\textbf{Low-rank}}   \\\cline{2-5}
                &  \textbf{model}  & \textbf{complexity} & \textbf{Algo}  & \textbf{parallel}   \\\hline
                &                    &   &  &    \\
                &    $\min_{X} \|M_r X M_c - \tilde{X}\|^{2}_{F}$   & $\mathcal{O}(n^3)$  & --  & --   \\
  ideal         &   $\text{s.t:} ~ X^\top \in span(Q_{k_c})$       &   &  &   \\
                &   $ ~ X \in span(P_{k_r})$                       &    &   &    \\\hline
                
                &                 &    &   &   \\
                &  $\min_{X} \|M_r X M_c - \tilde{X}\|^{2}_{F}$    &  $\mathcal{O}(Inp\K)$  & gradient   &  no  \\
 alter-      &    $+\gamma_c\tr(X\Larg_c X^\top) $    &    & descent  &   \\
    nate            &   $+\gamma_r\tr(X^\top \Larg_r X) $   &    &   &   \\\hline
                
                &   &    &   &  \\
                &  $\min_{U}  \|M_r U - \tilde{U}\|^{2}_{F}$  &  $\mathcal{O}(In\K)$  & gradient  & yes  \\
                &   $+ {\gamma^{'}}_r\tr(U^{\top} \Larg_r U)$   &    & descent  &   \\
  approx-  &  &    &   &  \\
    imate              &  $\min_{V}  \|M^{\top}_c V - \tilde{V}\|^{2}_{F}$  & $\mathcal{O}(Ip\K)$   & gradient  &  yes \\
                &   $+ {\gamma^{'}}_c\tr(V^{\top} \Larg_c V)$   &    &  descent &   \\
                 &  &    &   &    \\
                 &  $X = U \tilde{\Sigma}V^{\top}$  &    $\mathcal{O}(\rho^2_r\rho_c)$  & SVD  &  \\\hline
                                 &   &    &   &  \\
                &  $\min_{U}  \tr(U^\top \Larg_r U)$  & $\mathcal{O}(pkO_l\K)$   & PCG & \\
                &   $\text{s.t:} M_r U = \tilde{U}$   &  &    &  \\
  Subspace- &  &    &   &     \\
    Upsampling              &  $\min_{V}  \tr(V^\top \Larg_c V)$  &   $\mathcal{O}(nkO_l\K)$   & PCG   & yes  \\
                &   $\text{s.t:} M^{\top}_c V = \tilde{V}$   &    &   &   \\
                 &  &    &   &   \\
                 &  $X = U \tilde{\Sigma}V^{\top}$  &  $\mathcal{O}(\rho^2_r\rho_c)$  & SVD  &  \\\hline
                 
                 &  &    &   &    \\
                                 &  $\min_{U}  \tr(U^\top \Larg_r U)$  &  $\mathcal{O}(pkO_l\K)$  & PCG & yes \\
                &   $\text{s.t:} M_r U = \tilde{U}$   &    &  &  \\
  approx- &  &    &   &    \\
    imate 2              &  $X = UU^{\top}Y$  &    &   &  \\\hline
                   
                                    &  &    &   &   \\
                                 &  $\min_{V}  \tr(V^\top \Larg_c V)$  & $\mathcal{O}(nkO_l\K)$   &  PCG  & yes \\
                &   $\text{s.t:} M^{\top}_c V = \tilde{V}$   &    &   &  \\
  approx- &  &    &   &    \\
imate 3                  &  $X = YVV^{\top}$  &    &   & \\\hline
              
\end{tabular}}  
\label{tab:decoders}
\end{table}

\begin{sidewaystable}[htbp]
\caption{Computational complexity of all the models considered in this work}
%\centering
\resizebox{1.0\textwidth}{!}{\begin{tabular}[t]{| c || c | c | c | c | c | c | c | c | c |}
\hline
\textbf{Model}  & \textbf{Complexity } & \textbf{Complexity} & \textbf{Complexity } &  \multicolumn{3}{c|}{\textbf{Overall Complexity (low-rank)}} & \multicolumn{3}{c|}{\textbf{Overall Complexity (clustering)}} \\\cline{5-10}
 & $G_c$ & $G_r$  & \textbf{Algorithm} & \textbf{Fast SVD} & \textbf{Decoder}  & \textbf{Total}  & \textbf{k-means} & \textbf{Decoder}  & \textbf{Total} \\
                 &  $\mathcal{O}(np \log(n))$  & $\mathcal{O}(np\log(p))$  &for $p\ll n$  & for  $p\ll n$  &   &   & \cite{elkan2003using} &  & \\\hline\hline
LE  \cite{belkin2003laplacian}   &  + & -- & $\mathcal{O}(n^{3})$  & --  & --  & --  & $\mathcal{O}(Inkp)$  & --  &   $\mathcal{O}(n^3)$ \\\hline 
LLE  \cite{roweis2000nonlinear}     &  --  & --   & $\mathcal{O}((p+k)n^2)$  & -- & --  & --  & $\mathcal{O}(Inkp)$  & -- & $\mathcal{O}(pn^2)$  \\\hline
PCA   &  -- & -- & $\mathcal{O}(p^{2}n)$ &  --  & --   & $\mathcal{O}(n(p^2\log(n)+p^2))$ & $\mathcal{O}(Inkp)$ &  --  &  $\mathcal{O}(np^2\log(n)+np^2)$\\\hline
 GLPCA \cite{jiang2013graph} &  + & -- &  $\mathcal{O}(n^{3})$ & -- & --   & $\mathcal{O}(n(p\log(n)+n^2))$   & $\mathcal{O}(Inkp)$  & -- & $\mathcal{O}(n^3))$ \\\hline
 NMF  \cite{lee1999learning}  & --  & -- & $\mathcal{O}(Inpk)$ & --  & --  & $\mathcal{O}(Inpk)$ &  $\mathcal{O}(Inkp)$ & -- & $\mathcal{O}(Inp(k+K))$ \\\hline
 GNMF  \cite{cai2011graph} & + & -- & $\mathcal{O}(Inpk)$ & --  & --  & $\mathcal{O}(np \log(n))$ & $\mathcal{O}(Inkp)$   &  --   & $\mathcal{O}(np \log(n))$ \\\hline
 MMF  \cite{zhang2013low} & + & -- & $\mathcal{O}(((p+k)k^{2}+pk)I)$ & --  & -- & $\mathcal{O}(np \log(n))$ &  $\mathcal{O}(Inkp)$ & -- & $\mathcal{O}(np \log(n))$ \\\hline
 RPCA \cite{candes2011robust} & -- & --  & $\mathcal{O}(Inp^{2})$ & --  & --  & $\mathcal{O}(np^2I + n\log(n))$  & $\mathcal{O}(Inkp)$ & -- & $\mathcal{O}(np^2I+n\log(n))$ \\\hline
 RPCAG \cite{shahid2015robust}& -- & --  & $\mathcal{O}(Inp^{2})$ & --  & --  & $\mathcal{O}(np^2I+n\log(n))$  & $\mathcal{O}(Inkp)$ & -- & $\mathcal{O}(np^2I+n\log(n))$ \\\hline
 FRPCAG \cite{shahid2015fast} & + & + & $\mathcal{O}(Inp\K) $ & -- & --  & $\mathcal{O}(np \log(n))$   & $\mathcal{O}(Inkp)$ & --  & $\mathcal{O}(np \log(n))$\\\hline
 CPCA   & + & + & $\mathcal{O}(I \rho_c \rho_r \K) $ & $\mathcal{O}(\rho_c {\rho^2}_r)$ & $\mathcal{O}(nk O_l\K)$  &  $\mathcal{O}(np \log(n))$  &  $\mathcal{O}(I\rho_c k\rho_r)$ & $\mathcal{O}(nk O_l)$ & $\mathcal{O}(np \log(n))$\\\hline
\end{tabular}}
\label{tab:complexity}
\end{sidewaystable}

%\clearpage
\newpage
\onecolumn
\begin{table}[htbp]
\begin{minipage}[b]{1.0\textwidth}
\caption{Details of the datasets used for clustering experiments in this work.}
\centering
\resizebox{0.4\textwidth}{!}{\begin{tabular}[t]{| c | c | c | c |}
\hline
 \textbf{Dataset}  & \textbf{Samples}  & \textbf{Dimension}  & \textbf{Classes} \\\hline
 ORL  large  &  400   & $56 \times 46$   & 40  \\\hline
 ORL small  & 300   &  $28 \times 23$  &  30  \\\hline
 CMU PIE   &  1200  & $32 \times 32$  & 30  \\\hline
 YALE   & 165  &   $32 \times 32$  & 11  \\\hline
 MNIST   & 70000  & $28 \times 28$  & 10  \\\hline
 MNIST small & 1000  & $28 \times 28$  & 10  \\\hline
 USPS large  &  10000  &  $16 \times 16$  & 10  \\\hline
 USPS small  &  3500  &  $16 \times 16$  & 10  \\\hline
\end{tabular}}
\label{tab:datasets}
\end{minipage}
\begin{minipage}[b]{0.7\linewidth}
\centering
\caption{Range of parameter values for each of the models considered in this work. $k$ is the rank or dimension of subspace or the number of clusters, $\lambda$ is the weight associated with the sparse term for Robust PCA framework \cite{candes2011robust} and $\gamma, \alpha$ are  the parameters associated with the graph regularization term.}
\centering
\resizebox{1.0\textwidth}{!}{\begin{tabular}[t]{| c | c | c | } \hline
  \textbf{Model}   & \textbf{Parameters}   & \textbf{Parameter Range} \\\hline
LLE  \cite{roweis2000nonlinear}, PCA   & $k$ &  $k\in \{2^{1},2^{2},\cdots, \min(n,p) \}$  \\
 LE \cite{belkin2003laplacian}     &         &   \\\hline
  GLPCA  \cite{jiang2013graph}  &  & $k\in \{2^{1},2^{2},\cdots, \min(n,p) \}$    \\
   & $k,\gamma$ & $\gamma \implies \beta$ using  \cite{jiang2013graph} $\beta \in \{0.1, 0.2, \cdots, 0.9\}$ \\\hline
MMF  \cite{zhang2013low}      & $k,\gamma$ &  $k\in \{2^{1},2^{2},\cdots, \min(n,p) \}$ \\\cline{1-2}
 NMF  \cite{lee1999learning} &  $k$ &   \\\cline{1-2}
 GNMF \cite{cai2011graph}     & $k,\gamma$ &  $\gamma \in \{2^{-3},2^{-2},\cdots, 2^{10}\}$    \\\hline
 RPCA \cite{candes2011robust}  & $\lambda$  & $\lambda \in \{\frac{2^{-3}}{\sqrt{\max(n,p)}}:0.1:\frac{2^{3}}{\sqrt{\max(n,p)}}\}$ \\\cline{1-2}
 RPCAG \cite{shahid2015robust} & $\lambda, \gamma$ & $\gamma \in \{2^{-3},2^{-2},\cdots, 2^{10}\}$\\\hline   
  FRPCAG \cite{shahid2015fast} &   $\gamma_{r}, \gamma_{c}$ & $\gamma_{r}, \gamma_{c} \in (0, 30)$    \\\hline
   CPCA &   $\gamma_{r}, \gamma_{c}$ & $\gamma_{r}, \gamma_{c} \in (0, 30)$    \\\cline{2-3}
        &    $k$ (approximate decoder)  & $\tilde{\Sigma}_{k,k}/\tilde{\Sigma}_{1,1} < 0.1$  \\\hline
\end{tabular}}\label{tab:models_param}
\end{minipage}
\end{table}

\FloatBarrier

\end{document}